\documentclass{article}

\PassOptionsToPackage{numbers}{natbib}

\usepackage[final]{neurips_2018}

\usepackage[utf8]{inputenc} 
\usepackage[T1]{fontenc}    
\usepackage{hyperref}       
\usepackage{url}            
\usepackage{booktabs}       
\usepackage{amsfonts}       
\usepackage{nicefrac}       
\usepackage{microtype}      

\usepackage{amsmath}
\usepackage{amsthm}
\usepackage{nccmath}
\usepackage{amsfonts}
\usepackage{amssymb}
\usepackage{bm}
\usepackage{bbm}
\usepackage{mathtools}

\usepackage[dvipsnames]{xcolor}
\usepackage{pgfplots}
\usepackage{tikz}
\tikzstyle{every picture}+=[remember picture]
\usetikzlibrary{arrows,automata,calc}
\pgfplotsset{compat=newest}

\usepackage{graphicx}
\usepackage{subcaption}
\usepackage{enumitem}
\usepackage{xfrac}

\usepackage[colorinlistoftodos, textwidth=20mm]{todonotes}
\definecolor{citrine}{rgb}{0.89, 0.82, 0.04}
\definecolor{blued}{RGB}{70,197,221}

\usepackage{algorithm}                                                          
\usepackage{algorithmic}
\newcommand{\M}{\mathcal M}
\newcommand{\A}{\mathcal A}

\newcommand{\calS}{\mathcal S}
\newcommand{\calScom}{\mathcal{S}^{\texttt{C}}}
\newcommand{\Scom}{S^{\texttt{C}}}
\newcommand{\calStrans}{\mathcal{S}^{\texttt{T}}}
\newcommand{\Strans}{S^{\texttt{T}}}

\renewcommand{\Re}{\mathbb R}

\newcommand{\nextstates}{\Gamma}
\newcommand{\nextstatescom}{\Gamma^{\texttt{C}}}
\newcommand{\diamcom}{D^{\texttt{C}}}

\newcommand{\SP}[1]{sp_{\mathcal{S}}\left\{{#1}\right\}}

\DeclareMathOperator*{\argmax}{\arg\,\max}
\DeclareMathOperator*{\argmin}{\arg\,\min}

\newcommand{\transp}{\mathsf{T}}

\newcommand{\evi}{{\small\textsc{EVI}}\xspace}

\newcommand{\regal}{{\small\textsc{Regal}}\xspace}
\newcommand{\regalc}{{\small\textsc{Regal.C}}\xspace}
\newcommand{\regald}{{\small\textsc{Regal.D}}\xspace}

\newcommand{\ucrl}{{\small\textsc{UCRL}}\xspace}
\newcommand{\psrl}{{\small\textsc{PSRL}}\xspace}
\newcommand{\scal}{{\small\textsc{SCAL}}\xspace}
\newcommand{\tucrl}{{\small\textsc{TUCRL}}\xspace}

\newcommand{\tevi}{{\small\textsc{TEVI}}\xspace}

\newcommand{\rmaxbound}{r_{\max}}

\newcommand{\wt}[1]{\widetilde{#1}}
\newcommand{\wh}[1]{\widehat{#1}}

\newcommand{\wb}[1]{\overline{#1}}

\usepackage{accents}
\DeclareMathAccent{\wtilde}{\mathord}{largesymbols}{"65}

\newcommand{\pluseq}{\mathrel{+}=}

\DeclareRobustCommand{\eg}{e.g.,\@\xspace}
\DeclareRobustCommand{\ie}{i.e.,\@\xspace}
\DeclareRobustCommand{\wrt}{w.r.t.\@\xspace}

\DeclareRobustCommand{\st}{s.t.\@\xspace}

\newtheorem{theorem}{Theorem}
\newtheorem{lemma}{Lemma}

\theoremstyle{remark}

\theoremstyle{theorem}

\newtheorem{example}{Example}
\newtheorem{assumption}{Assumption}

\newlength{\minipagewidth}
\newlength{\minipagewidthx}
\newcommand{\bookboxx}[1]{\small
\par\medskip\noindent
\framebox[0.99\textwidth]{
\begin{minipage}{0.97\dimexpr\textwidth-\parindent\relax} {#1} \end{minipage} } \par\medskip }

\usepackage{xspace}                                                             
\DeclareRobustCommand{\eg}{e.g.,\@\xspace}                                      
\DeclareRobustCommand{\ie}{i.e.,\@\xspace}                                      
\DeclareRobustCommand{\wrt}{w.r.t.\@\xspace}  

\title{
        Near Optimal Exploration-Exploitation in Non-Communicating Markov Decision Processes
}

\author{
  Ronan Fruit \\
  Sequel Team - Inria Lille \\
  \texttt{ronan.fruit@inria.fr}\\
  \And
  Matteo Pirotta \\
  Sequel Team - Inria Lille \\
  \texttt{matteo.pirotta@inria.fr}\\
  \And
  Alessandro Lazaric \\
  Facebook AI Research\\
  \texttt{lazaric@fb.com}\\
}

\begin{document}

\maketitle

\begin{abstract}
While designing the state space of an MDP, it is common to include states that are transient or not reachable by any policy (\eg in mountain car, the product space of speed and position contains configurations that are not physically reachable). This results in weakly-communicating or multi-chain MDPs.
In this paper, we introduce \tucrl, the first algorithm able to perform efficient exploration-exploitation in any finite Markov Decision Process (MDP) without requiring any form of prior knowledge.
In particular, for any MDP with $\Scom$ communicating states, $A$ actions and $\nextstatescom \leq \Scom$ possible communicating next states, we derive a $\wt{O}(\diamcom\sqrt{\nextstatescom \Scom AT})$ regret bound, where $\diamcom$ is the diameter (\ie the length of the longest shortest path between any two states) of the communicating part of the MDP.
This is in contrast with existing optimistic algorithms (\eg~\ucrl, Optimistic \psrl) that suffer linear regret in weakly-communicating MDPs, as well as posterior sampling or regularised algorithms (\eg~\regal), which require prior knowledge on the bias span of the optimal policy to achieve sub-linear regret. We also prove that in weakly-communicating MDPs, no algorithm can ever achieve a logarithmic growth of the regret without first suffering a linear regret for a number of steps that is exponential in the parameters of the MDP. Finally, we report numerical simulations supporting our theoretical findings and showing how \tucrl overcomes the limitations of the state-of-the-art.
\end{abstract}


\vspace{-.3cm}
\section{Introduction}
\vspace{-.2cm}
Reinforcement learning (RL)~\citep{sutton1998reinforcement} studies the problem of learning in sequential decision-making problems where the dynamics of the environment is unknown, but can be learnt by performing actions and observing their outcome in an online fashion.
A sample-efficient RL agent must trade off the \emph{exploration} needed to collect information about the environment, and the \emph{exploitation} of the experience gathered so far to gain as much reward as possible.
In this paper, we focus on the regret framework in \emph{infinite-horizon average-reward} problems~\citep{Jaksch10}, where the exploration-exploitation performance is evaluated by comparing the rewards accumulated by the learning agent and an optimal policy.
\citet{Jaksch10} showed that it is possible to efficiently solve the exploration-exploitation dilemma using the \emph{optimism in face of uncertainty} (OFU) principle.
OFU methods build confidence intervals on the dynamics and reward (\ie construct a set of plausible MDPs), and execute the optimal policy of the ``best'' MDP in the confidence region~\citep[\eg][]{Jaksch10,Bartlett2009regal,fruit2017optionsnoprior,Talebi2018variance,fruit2018constrained}.
An alternative approach is posterior sampling (PS)~\citep{thompson1933likelihood}, which maintains a posterior distribution over MDPs and, at each step, samples an MDP and executes the corresponding optimal policy~\citep[\eg][]{Osband2013more,Abbasi-Yadkori2015bayesian,Osband2017posterior,Ouyang2017learning,Agrawal2017posterior}.

\begin{figure}[t]
\begin{subfigure}{0.7\textwidth}
         \centering
        \includegraphics[width=\textwidth]{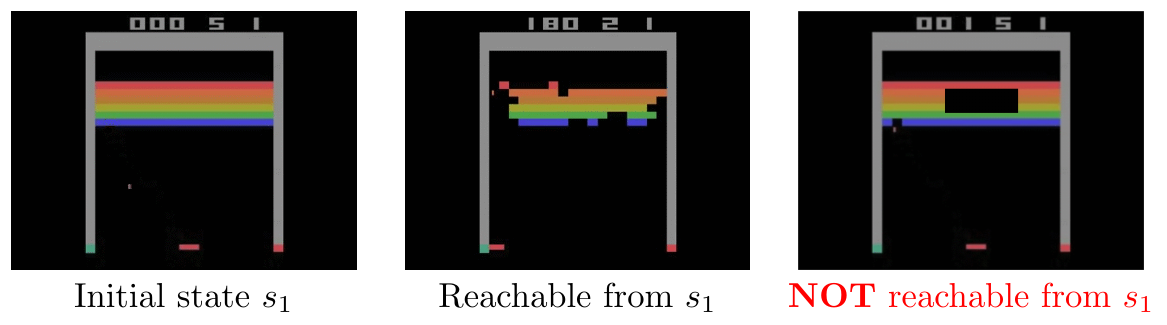}
        \caption{Breakout}
        \label{fig:breakout}
 \end{subfigure}
 \begin{subfigure}{0.28\textwidth}
         \centering
         \includegraphics[width=\textwidth]{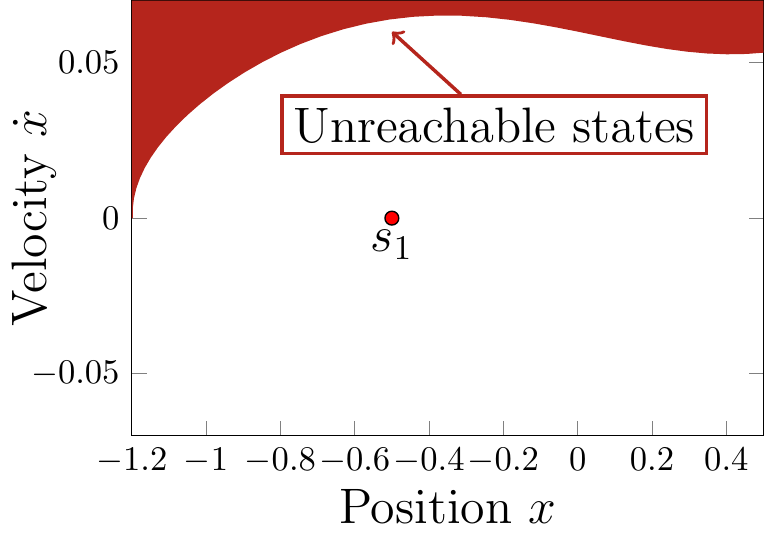}
        \caption{Mountain Car}
        \label{fig:mountain}
 \end{subfigure}
 \caption{Examples of non-communicating domains. Fig.~\subref{fig:mountain} represents a phase plane plot of the Mountain car domain $(x,\dot{x}) \in [-1.2,0.6] \times [-0.07,0.07]$. The initial state is $(-0.5,0)$ and the red area corresponds to non-reachable states from the initial state. Other non-reachable states may exist. Fig.~\subref{fig:breakout} shows the initial state, one reachable state (\emph{middle}) and an unreachable one (\emph{right}).}
 \vspace{-0.5cm}
\end{figure}

\textbf{Weakly-communicating MDPs and misspecified states.} One of the main limitations of \ucrl~\citep{Jaksch10} and optimistic \psrl~\citep{Agrawal2017posterior} is that they require the MDP to be communicating so that its diameter $D$ (\ie the length of the longest path among all shortest paths between any pair of states) is finite.
While assuming that all states are reachable may seem a reasonable assumption, it is rarely verified \textit{in practice}.
In fact, it requires a designer to carefully define a state space $\calS$ that contains all reachable states (otherwise it may not be possible to learn the optimal policy), but it excludes unreachable states (otherwise the resulting MDP would be non-communicating).
This requires a considerable amount of prior knowledge about the environment. 
%
        Consider a problem where we learn from images \eg the Atari Breakout game \citep{Mnih15}.
        The state space is the set of ``plausible'' configurations of the brick wall, ball and paddle positions.
        The situation in which the wall has an hole in the middle is a valid state (\eg as an initial state) but it cannot be observed/reached starting from a dense wall (see Fig.~\ref{fig:breakout}). As such, it should be removed to obtain a ``well-designed'' state space. While it may be possible to design a suitable set of ``reachable'' states that define a communicating MDP, this is often a difficult and tedious task, sometimes even impossible.
        Now consider a continuous domain \eg the Mountain Car problem \citep{Moore90}. The state is decribed by the position $x$ and velocity $\dot{x}$ along the $x$-axis. The state space of this domain is usually defined as the cartesian product $[-1.2,0.6] \times [-0.07,0.07]$. Unfortunately, this set contains configurations that are not physically reachable as shown on Fig.~\ref{fig:mountain}. The \emph{dynamics} of the system is constrained by the \emph{evolution equations}. Therefore, the car can not go arbitrarily fast.
        On the leftmost position ($x=-1.2$) the speed $\dot{x}$ cannot exceed $0$ due to the fact that such position can be reached only with velocity $\dot{x} \leq 0$.
        To have a higher velocity, the car would need to acquire momentum from further left (\ie $x< -1.2$) which is impossible by design ($-1.2$ is the left-boundary of the position domain).
        The maximal speed reachable for $x>-1.2$ can be attained by applying the maximum acceleration at any time step starting from the state $(x,\dot{x}) =(-1.2,0)$. This identifies the curve reported in the Fig.~\ref{fig:mountain} which denotes the boundary of the unreachable region. Note that other states may not be reachable.
Whenever the state space is \textit{misspecified} or the MDP is weakly communicating (\ie $D=+\infty$), OFU-based algorithms (\eg \ucrl) optimistically attribute large reward and non-zero probability to reach states that have never been observed, and thus they tend to repeatedly attempt to \textit{explore} unreachable states. This results in poor performance and linear regret. A first attempt to overcome this major limitation is \regalc~\citep{Bartlett2009regal} (\citet{fruit2018constrained} recently proposed \scal, an implementable efficient version of \regalc), which requires prior knowledge of an upper-bound $H$ to the span (\ie range) of the optimal bias function $h^*$. The optimism of \ucrl is then ``constrained'' to policies whose bias has span smaller than $H$. 
This implicitly ``removes'' non-reachable states, whose large optimistic reward would cause the span to become too large. Unfortunately, an accurate knowledge of the bias span may not be easier to obtain than designing a well-specified state space. \citet{Bartlett2009regal} proposed an alternative algorithm -- \regald -- that leverages on the \emph{doubling trick}~\citep{auer1995gambling}
to avoid any prior knowledge on the span. Nonetheless, we recently noticed a major flaw in the proof of \citep[Theorem 3]{Bartlett2009regal} that questions the validity of the algorithm (see App.~\ref{app:mistake_regal} for further details).
PS-based algorithms also suffer from similar issues.\footnote{We notice that the problem of weakly-communicating MDPs and misspecified states does not hold in the more restrictive setting of finite horizon~\citep[\eg][]{Osband2013more} since exploration is directly tailored to the states that are reachable \textit{within} the known horizon, or under the assumption of the existence of a recurrent state~\citep[\eg][]{DBLP:conf/colt/GopalanM15}.}
To the best of our knowledge, the only regret guarantees available in the literature for this setting are \citep{Abbasi-Yadkori:2015:BOC:3020847.3020849, NIPS2017_6732, theocharous2017largeps}. However, the counter-example of \citet{DBLP:journals/corr/OsbandR16a} seems to invalidate the result of \citet{Abbasi-Yadkori:2015:BOC:3020847.3020849}. On the other hand, \citet{NIPS2017_6732} and \citet{theocharous2017largeps} present PS algorithms with expected \emph{Bayesian} regret scaling linearly with $H$, where $H$ is an upper-bound on the optimal bias spans of all the MDPs that can be drawn from the prior distribution (\citep[Asm. 1]{NIPS2017_6732} and \citep[Sec. 5]{theocharous2017largeps}).
In \citep[Remark 1]{NIPS2017_6732}, the authors claim that their algorithm does not require the knowledge of $H$ to derive the regret bound. However, in App.~\ref{app:unbounded.span.bayesian} we show on a very simple example that for most continuous prior distributions (\eg uninformative priors like Dirichlet), it is very likely that $H =+ \infty$ implying that the regret bound may not hold (similarly for \citep{theocharous2017largeps}).
As a result, similarly to \regalc, the prior distribution should contain prior knowledge on the bias span to avoid poor performance.

In this paper, we present \tucrl, an algorithm designed to trade-off exploration and exploitation in weakly-communicating and multi-chain MDPs (e.g., MDPs with misspecified states) without any prior knowledge and under the only assumption that the agent starts from a state in a communicating subset of the MDP (Sec.~\ref{sec:tucrl}).
In communicating MDPs, \tucrl eventually (after a finite number of steps) performs as \ucrl, thus achieving problem-dependent logarithmic regret.
When the true MDP is weakly-communicating, we prove that \tucrl achieves a $\wt{O}(\sqrt{T})$ regret that with polynomial dependency on the MDP parameters.
We also show that it is not possible to design an algorithm achieving logarithmic regret in weakly-communicating MDPs without having an exponential dependence on the MDP parameters (see Sec.~\ref{sec:no_free_lunch}).
\tucrl is the first computationally tractable algorithm in the OFU literature that is able to adapt to the MDP nature without any prior knowledge.
The theoretical findings are supported by experiments on several domains (see Sec.~\ref{sec:experiments}).


\vspace{-.2cm}
\section{Preliminaries}\label{sec:preliminaries}
\vspace{-.2cm}

We consider a finite \emph{weakly-communicating} Markov decision process \citep[Sec. 8.3]{puterman1994markov} $M = \langle \calS, \A, r, p\rangle$ with a set of states $\calS$ and a set of actions $\A=\bigcup_{s \in \calS} \A_s$. Each state-action pair $(s,a)\in \mathcal{S}\times \mathcal{A}_s$ is characterized by a reward distribution with mean $r(s,a)$ and support in $[0, \rmaxbound]$ as well as a transition probability distribution $p(\cdot|s,a)$ over next states.
In a weakly-communicating MDP, the state-space $\calS$ can be \emph{partioned} into two subspaces \citep[Section 8.3.1]{puterman1994markov}: a \emph{communicating} set of states (denoted $\calScom$ in the rest of the paper) with each state in $\Scom$ accessible --with non-zero probability-- from any other state in $\Scom$ under some stationary deterministic policy, and a --possibly empty-- set of states that are \emph{transient} under all policies (denoted $\calStrans$).
We also denote by $S = |\calS|$, $\Scom = |\calScom|$ and $A = \max_{s\in\mathcal{S}}|\A_s|$ the number of states and actions, and by $\nextstatescom = \max_{s\in\calScom,a\in\A}\|p(\cdot|s,a)\|_0$ the maximum support of all transition probabilities $p(\cdot|s,a)$ with $s \in \calScom$.
The sets $\calS^{\texttt{C}}$ and $\calS^{\texttt{T}}$ form a partition of $\calS$ \ie $\calScom \cap \calStrans = \emptyset$ and $\calScom \cup \calStrans = \calS$. 
A deterministic policy $\pi : \calS \rightarrow \A$ maps states to actions and it has an associated \emph{long-term average reward} (or \textit{gain}) and a \textit{bias function} defined as
\begin{align*}
        g^\pi_M(s) := \lim_{T\to \infty} \mathbb{E}\bigg[ \frac{1}{T}\sum_{t=1}^T r\big(s_t,\pi(s_t)\big) \bigg]; \quad h^\pi_M(s) := \underset{T\to \infty}{C\text{-}\lim}~\mathbb{E}\bigg[\sum_{t=1}^{T} \big(r(s_t,\pi(s_t)) - g_M^\pi(s_t)\big)\bigg],
\end{align*}
where the bias $h^\pi_M(s)$ measures the expected total difference between the rewards accumulated by $\pi$ starting from $s$ and the stationary reward in \emph{Cesaro-limit}\footnote{For policies whose associated Markov chain is aperiodic, the standard limit exists.} (denoted $C\text{-}\lim$). Accordingly, the difference of bias values $h^\pi_M(s)-h^\pi_M(s')$ quantifies the (dis-)advantage of starting in state $s$ rather than $s'$. In the following, we drop the dependency on $M$ whenever clear from the context and
denote by $\SP{h^\pi} := \max_{s\in\calS} h^\pi(s) - \min_{s\in\calS} h^\pi(s)$ the \emph{span} of the bias function. 
In weakly communicating MDPs, any optimal policy $\pi^* \in \argmax_\pi g^\pi(s)$ has \emph{constant} gain, \ie $g^{\pi^*}(s) = g^*$ for all $s\in\calS$.
Finally, we denote by $D$, resp. $\diamcom$, the diameter of $M$, resp. the diameter of the communicating part of $M$ (\ie restricted to the set $\calScom$):
\begin{align} \label{eq:diameters}
	D := \max_{(s,s') \in \calS\times \calS, s\neq s'} \{\tau_{M}(s \to s')\}, \qquad
        \diamcom := \max_{(s,s') \in \calScom \times \calScom, s\neq s'} \{\tau_{M}(s \to s')\}, 
\end{align}
where $\tau_{M}(s\to s')$ is the expected time of the shortest path from $s$ to $s'$ in $M$. 

\textbf{Learning problem.}
Let $M^*$ be the true (\emph{unknown}) weakly-communicating MDP.
We consider the learning problem where $\mathcal{S}$, $\mathcal{A}$ and $\rmaxbound$ are \emph{known}, while sets $\calScom$ and $\calStrans$, rewards $r$ and transition probabilities $p$ are \emph{unknown} and need to be estimated on-line. We evaluate the performance of a learning algorithm $\mathfrak{A}$ after $T$ time steps by its cumulative \emph{regret} $\Delta(\mathfrak{A},T) = T g^* - \sum_{t=1}^T r_t(s_t,a_t)$. Furthermore, we state the following assumption.
\begin{assumption}\label{asm:initial.state}
        The initial state $s_1$ belongs to the communicating set of states $\calScom$.
\end{assumption}
%
While this assumption somehow restricts the scenario we consider, it is fairly common in practice.
For example, all the domains that are characterized by the presence of a resetting distribution (\eg episodic problems) satisfy this assumption (\eg mountain car, cart pole, Atari games, taxi, etc.).

\textbf{Multi-chain MDPs.} While we consider weakly-communicating MDPs for ease of notation, all our results extend to the more general case of multi-chain MDPs.\footnote{In the case of misspecified states, we implicitly define a multi-chain MDP, where each non-reachable state has a self-loop dynamics and it defines a ``singleton'' communicating subset.} In this case, there may be multiple communicating and transient sets of states and the optimal gain $g^*$ is different in each communicating subset. In this case we define $\Scom$ as the set of states that are accessible --with non-zero probability-- from $s_1$ ($s_1$ included) under some stationary deterministic policy. $\calStrans$ is defined as the complement of $\Scom$ in $\calS$ \ie $\calStrans := \calS \setminus \Scom$. With these new definitions of $\Scom$ and $\calStrans$, Asm.~\ref{asm:initial.state} needs to be reformulated as follows:

        \vspace{.1cm}
        \textbf{Assumption 1 for Multi-chain MDPs.}
        {\itshape
The initial state $s_1$ is accessible --with non-zero probability-- from any other state in $\calScom$ under some stationary deterministic policy. Equivalently, $\calScom$ is a communicating set of states.}
        \vspace{.1cm}

Note that the states belonging to $\calStrans$ can either be transient or belong to other communicating subsets of the MDP disjoint from $\calScom$. It does not really matter because the states in $\calStrans$ will never be visited by definition.
As a result, the regret is still defined as before, where the learning performance is compared to the optimal gain $g^*(s_1)$ related to the communicating set of states $\calScom\ni s_1$.


\vspace{-.2cm}
\section{Truncated Upper-Confidence for Reinforcement Learning (\tucrl)}\label{sec:tucrl}
\vspace{-.2cm}
In this section we introduce Truncated Upper-Confidence for Reinforcement Learning (\tucrl), an optimistic online RL algorithm that efficiently balances exploration and exploitation to learn in non-communicating MDPs without prior knowledge (Fig.~\ref{fig:ucrl.constrained}).

Similar to \ucrl, at the beginning of each episode $k$, \tucrl constructs confidence intervals for the reward and the dynamics of the MDP. Formally, for any $(s,a) \in \calS \times \A$ we define
\begin{align}\label{eq:Bp.1plus}
        B_{p,k}(s,a) 
        &= \left\{ 
                \wt{p}(\cdot|s,a)\in \mathcal{C} :~ \forall s'\in\calS,
                |\wt{p}(s'|s,a) - \wh{p}(s'|s,a) | \leq \beta_{p,k}^{sas'}
        \right\},\\ 
	B_{r,k}(s,a) &:= [\wh{r}_k(s,a) - \beta_{r,k}^{sa}, \wh{r}_k(s,a) + \beta_{r,k}^{sa}] \cap [0, \rmaxbound],
\end{align}
where $\mathcal{C} = \{p \in \mathbb{R}^S | \forall s',~p(s') \geq 0 \wedge \sum_{s'} p(s') = 1\}$ is the $(S-1)$-probability simplex, while the size of the confidence intervals is constructed using the empirical Bernstein's inequality~\citep{audibert2007tuning,Maurer2009empirical} as
\begin{align*}
\beta_{r,k}^{sa}
:=
         \sqrt{\frac{14 
                         \wh{\sigma}_{r,k}^2(s,a)
                        b_{k,\delta}
                }{
        N^{+}_k(s,a)
}} + \frac{\frac{49}{3}
\rmaxbound
        b_{k,\delta}
        }{
        N^{\pm}_k(s,a)
}
, \qquad
\beta_{p,k}^{sas'}
:= 
\sqrt{\frac{14 \wh{\sigma}_{p,k}^2(s'|s,a) 
                        b_{k,\delta}
                }{
        N^{+}_k(s,a)
}
} +\frac{\frac{49}{3}
        b_{k,\delta}
        }{
        N^{\pm}_k(s,a)
},
\end{align*}
where $N_k(s,a)$ is the number of visits in $(s, a)$ before episode $k$, 
$N^{+}_k(s,a) := \max\lbrace 1, N_k(s,a) \rbrace$, 
$N^{\pm}_k(s,a) := \max\lbrace 1, N_k(s,a)-1 \rbrace$, 
$\wh{\sigma}_{r,k}^2(s,a)$ and $\wh{\sigma}_{p,k}^2(s'|s,a)$ are the empirical variances of $r(s,a)$ and ${p}(s'|s,a)$ 
and $b_{k,\delta} = \ln(2 SA t_k/\delta)$. The set of plausible MDPs associated with the confidence intervals is then ${\mathcal{M}}_k = \big\{ M = (\calS,\A, \wt{r}, \wt{p}) :\; \wt{r}(s,a) \in B_{r,k}(s,a), \; \wt{p}(\cdot|s,a) \in B_{p,k}(s,a) \big\}$. \ucrl is optimistic w.r.t.\ the confidence intervals so that for all states $s$ that have never been visited the optimistic reward $\wt{r}(s,a)$ is set to $\rmaxbound$, while all transitions to $s$ (i.e., $\wt p(s|\cdot,\cdot)$) are set to the largest value compatible with $B_{p,k}(\cdot,\cdot)$. Unfortunately, some of the states with $N_k(s,a)=0$ may be actually unreachable (i.e., $s\in\calStrans$) and \ucrl would 
uniformly explore the policy space with the hope that at least one policy reaches those
(optimistically desirable) states. \tucrl addresses this issue by first constructing empirical estimates of $\calScom$ and $\calStrans$ (\ie the set of communicating and transient states in $M^*$) using the states that have been visited so far, that is $\calScom_k := \left\{s \in \calS \;\big|\; \sum_{a \in \A_s} N_k(s,a) > 0 \right\} \cup \{s_{t_k}\}$ and $\calStrans_k := \calS \setminus \calScom_k$, where $t_k$ is the starting time of episode $k$. 

In order to avoid optimistic exploration attempts to unreachable states, we could simply execute \ucrl on $\calScom_k$, which is guaranteed to contain only states in the communicating set (since $s_1\in\calScom$ by Asm.~\ref{asm:initial.state}, we have that $\calScom_k \subseteq \calScom$). 
Nonetheless, this algorithm could \textit{under-explore} state-action pairs that would allow discovering other states in $\calScom$, thus getting stuck in a subset of the communicating states of the MDP and suffering linear regret. 
While the states in $\calScom_k$ are guaranteed to be in the communicating subset, it is not possible to know whether states in $\calStrans_k$ are actually reachable from $\calScom_k$ or not. 
Then \tucrl first ``guesses'' a lower bound on the probability of transition from states $s\in\calScom_k$ to $s'\in\calStrans_k$ and whenever the maximum transition probability from $s$ to $s'$ compatible with the confidence intervals (i.e., $\wh{p}_k(s'|s,a) + \beta_{p,k}^{sas'}$) is below the lower bound, it assumes that such transition is not possible.
This strategy is based on the intuition that a transition either does not exist or it should have a sufficiently ``big'' mass. 
However, these transitions should be periodically reconsidered in order to avoid \emph{under-exploration} issues.
More formally, let $(\rho_t)_{t \in \mathbb{N}}$ be a non-increasing sequence to be defined later, for all $s'\in\calStrans_k$, $s\in\calScom_k$ and $a\in\A_s$, the empirical mean $\wh{p}_k(s'|s,a)$ and variance $\wh{\sigma}_{p,k}^2(s'|s,a)$ are zero (\ie this transition has never been observed so far), so the largest probability (most optimistic) of transition from $s$ to $s'$ through any action $a$ is $\wt{p}_k^+(s'|s,a) = \frac{49}{3}\frac{ b_{k,\delta}}{N^{\pm}_k(s,a)}$.
\tucrl compares $\wt{p}_k^+(s'|s,a)$ to $\rho_{t_k}$ and forces all transition probabilities below the threshold to zero, while the confidence intervals of transitions to states that have already been explored (i.e., in $\calScom_k$) are preserved unchanged. This corresponds to constructing the alternative confidence interval
\begin{align}\label{eq:Bp.tucrl.bern}
        \wb B_{p,k}(s,a) 
        &= B_{p,k}(s,a) \cap \left\{ 
                \wt{p}(\cdot|s,a)\in\mathcal{C}  :~ \forall s'\in\calStrans_k \text{ and } \wt{p}_k^+(s'|s,a) < \rho_{t_k}, \wt{p}(s'|s,a) = 0\right\}.
\end{align}
Given $\wb B_{p,k}$, \tucrl (implicitly) constructs the corresponding set of plausible MDPs $\wb{\mathcal{M}}_k$ and then solves the optimistic optimization problem 
\begin{equation}
        \label{eq:tucrl.optimization}
        (\wt{M}_k, \wt{\pi}_k) = \argmax_{M \in \wb{\mathcal{M}}_k, \pi} \{g^{\pi}_M\}.
\end{equation}
The resulting algorithm follows the same structure as \ucrl and it is shown in Fig.~\ref{fig:ucrl.constrained}. The episode stopping condition at line 4 is slightly modified w.r.t.\ \ucrl. In fact, it guarantees that one action is always executed and it forces an episode to terminate as soon as a state previously in $\calStrans_k$ is visited (i.e., $N_k(s_t,a) = 0$). This minor change guarantees that $N_{k+1}(s,a) = 0$ for all the states $s\in\calStrans_k$ that were not reachable at the beginning of the episode. The algorithm also needs minor modifications to the extended value iteration (\evi) algorithm used to solve~\eqref{eq:tucrl.optimization} to guarantee both efficiency and convergence. All technical details are reported in App.~\ref{app:algorithm}.

\begin{figure}[t]
\renewcommand\figurename{\small Figure}
\begin{minipage}{\columnwidth}
\bookboxx{
\textbf{Input:} Confidence $\delta \in ]0,1[$, $r_{\max}$, $\calS$, $\A$

\noindent \textbf{Initialization:} Set $N_0(s,a) := 0$ for any $(s,a) \in \calS \times \A$, $t := 1$ and observe $s_1$.

\noindent \textbf{For} episodes $k=1, 2, ...$ \textbf{do}

\begin{enumerate}[leftmargin=4mm,itemsep=0mm,topsep=0mm]
\item Set $t_k = t$ and episode counters $\nu_k (s,a) = 0$

\item Compute estimates $\wh{p}_k(s' | s,a)$, $\wh{r}_k(s,a)$ and a set
        $\wb{\mathcal{M}}_k$

\item
Compute an $\rmaxbound/\sqrt{t_k}$-approximation $\widetilde{\pi}_k$ of
 Eq.~\ref{eq:tucrl.optimization}

%

\item \textbf{While} 
        $t_k == t$ \textbf{or}  $\Big( \sum_{a \in \A_{s_t}} N_k(s_t,a) > 0 \textbf{ and } \nu_k(s_t, \wt{\pi}_k(s_t)) \leq \max \left\{ 1, N_k\left(s_t,\wt{\pi}_k(s_t)\right) \right\} \Big)$
        \textbf{do}
\begin{enumerate}[leftmargin=4mm,itemsep=0.5mm]
        \item Execute $a_t = \wt{\pi}_k(s_t)$, obtain reward $r_{t}$, and observe $s_{t+1}$
        \item Set $\nu_k (s_t,a_t) \pluseq 1$ and set $t \pluseq 1$ 
\end{enumerate}

\item Set $N_{k+1}(s,a) = N_{k}(s,a)+ \nu_k(s,a)$ \tikz[baseline]{\node[anchor=base](alg_lastline){};}
\end{enumerate}
}
\vspace{-.3cm}
\caption{\small \tucrl algorithm.}
\label{fig:ucrl.constrained}
\end{minipage}
\vspace{-.5cm}
\end{figure}

In practice, we set $\rho_t = \frac{49b_{t,\delta}}{3}\sqrt{\frac{SA}{t}}$, so that the \emph{condition to remove transition} reduces to $N_k^\pm(s,a) > \sqrt{\sfrac{t_k}{SA}}$. This shows that only transitions from state-action pairs that have been poorly visited so far are enabled, while if the state-action pair has already been tried often and yet no transition to $s'\in\calStrans_k$ is observed, then it is assumed that $s'$ is not reachable from $s,a$. When the number of visits in $(s, a)$ is big, the transitions to ``unvisited'' states should be discarded because if the transition actually exists, it is most likely extremely small and so it is worth exploring other parts of the MDP
first. Symmetrically, when the number of visits in $(s, a)$ is small, the transitions to ``unvisited'' states should be enabled 
because the transitions are quite plausible and the algorithm should try to explore the outcome of taking action $a$ in $s$ and possibly reach states in $\calStrans_k$.
We denote the set of state-action pairs that are not sufficiently explored by $\mathcal{K}_k = \big\{ (s,a) \in \calScom_k \times \A \,:\, N_k^\pm(s,a) \leq \sqrt{\sfrac{t_k}{SA}} \big\}$.

\vspace{-.1cm}
\subsection{Analysis of \tucrl} \label{sec:sketch.regret.proof}
\vspace{-.1cm}

We prove that the regret of \tucrl is bounded as follows.
\begin{theorem}\label{thm:tucrl.regret}
        For any weakly communicating MDP $M$, with probability at least $1-\delta$ it holds that for any $T > 1$, the regret of \tucrl is bounded as
        \[
                \Delta(\tucrl,T) = O \left(\rmaxbound \diamcom \sqrt{\nextstatescom \Scom A T \ln\left( \frac{SAT}{\delta} \right)} + \rmaxbound \Big(\diamcom \Big)^2 S^3 A \ln^2\left( \frac{SAT}{\delta} \right) \right).
        \]
\end{theorem}
The first term in the regret shows the ability of \tucrl to adapt to the communicating part of the true MDP $M^*$ by
scaling with the \emph{communicating} diameter $\diamcom$ and MDP parameters $\Scom$ and $\nextstatescom$. The second term corresponds to the regret incurred in the early stage where the regret grows linearly.
When $M^*$ is communicating, we match the square-root term of \ucrl (first term), while the second term is bigger than the one appearing in \ucrl by a multiplicative factor $\diamcom S$ (ignoring logarithmic terms, see Sec.~\ref{sec:no_free_lunch}).

We now provide a sketch of the proof of Thm.~\ref{thm:tucrl.regret} (the full proof is reported in App.~\ref{app:tucrl.regret}). In order to preserve readability, all following inequalities should be interpreted up to minor approximations and in high probability. 

Let $\Delta_k := \sum_{s,a}\nu_k(s,a)(g^* -r(s,a))$ be the regret incurred in episode $k$, where $\nu_k(s,a)$ is the number of visits to $s,a$ in episode $k$. We decompose the regret as
\begin{align*}
\Delta(\tucrl, T) \lesssim \sum_{k=1}^m \Delta_k \cdot \mathbbm{1}\{M^*\in \mathcal{M}_k\} \lesssim \sum_{k=1}^m \Delta_k \cdot \mathbbm{1}\{t_k < C(k)\} + \sum_{k=1}^m \Delta_k \cdot \mathbbm{1}\{t_k \geq C(k)\}
\end{align*}
where $C(k) = O\left((\diamcom)^2 S^3A\ln^2(\sfrac{2SAt_k}{\delta})\right)$ defines the length of a full exploratory phase, where the agent may suffer \textit{linear regret}.



\textbf{Optimism.} The first technical difficulty is that whenever some transitions are disabled, the plausible set of MDPs $\wb{\mathcal{M}}_k$ may actually be \textit{biased} and not contain the true MDP $M^*$. This requires to prove that \tucrl (i.e., the gain of the solution returned by \evi) is always optimistic despite ``wrong'' confidence intervals. 
The following lemma helps to identify the possible scenarios that \tucrl can produce (see App.~\ref{app:regret.Minset}).\footnote{Notice that $M^* \in \mathcal{M}_k$ is true w.h.p. since $\mathcal{M}_k$ is obtained using non-truncated confidence intervals.}
\begin{lemma}
        Let episode $k$ be such that $M^* \in \mathcal{M}_k$, $\calStrans_k\neq \emptyset$ and $t_k \geq C(k)$. 
        Then, either $\calStrans_k = \calStrans$ (case I) or $\mathcal{K}_k \neq \emptyset$, \ie$\exists(s,a) \in \calScom_k \times \A$ for which transitions to $\calStrans_k$ are allowed (case II).
\end{lemma}
This result basically excludes the case where $\calStrans_k \supset \calStrans$ (i.e., some states have not been reached)  and yet no transition from $\calScom_k$ to them is enabled. We start noticing that when $\calStrans_k = \emptyset$, the true MDP $M^* \in \mathcal{M}_k = \wb{\mathcal{M}}_k$ w.h.p. by construction of the confidence intervals.
Similarly, if $\calStrans_k = \calStrans$ then $M^* \in \wb{\mathcal{M}}_k$ w.h.p., since \tucrl only truncates transitions that are indeed forbidden in $M^*$ itself.
In both cases,  we can use the same arguments in~\citep{Jaksch10} to prove optimism.
In \emph{case II} the gain of any state $s' \in \calStrans_k$ is set to $\rmaxbound$ and, since there exists a path from $\calScom_k$ to $\calStrans_k$, the gain of the solution returned by EVI is $\rmaxbound$, which makes it trivially optimistic. As a result we can conclude that $\wt g_k \gtrsim g^*$ (up to the precision of \evi).

\textbf{Per-episode regret.} After bounding the optimistic reward $\wt{r}_k(s,a)$ w.r.t.\ $r(s,a)$, the only part left to bound the per-episode regret $\Delta_k$ is the term $\wt\Delta_k = \sum_{s,a}\nu_k(s,a)(\wt{g}_k -\wt{r}_k(s,a))$. Similar to \ucrl, we could use the (optimistic) optimality equation and rewrite $\wt\Delta_k$ as
\begin{align}\label{eq:no.good}
\begin{split}
        \wt{\Delta}_k = \sum_{s\in\textcolor{red}{\calS}}\nu_k(s,\wt{\pi}_k(s))\bigg( \sum_{s'\in \textcolor{red}{\calS}} \wt{p}_k(s'|s,\wt{\pi}_k(s)) \wt{h}_k(s') - \wt{h}_k(s) \bigg) = \nu_k' \left( \wt{P}_k -I \right) 
w_k
\end{split}
\end{align}
where $w_k := \wt{h}_k - \min_{s\in\calS}\{\wt{h}_k\} e$ is a shifted version of the vector $\wt{h}_k$ returned by EVI at episode $k$, and then proceed by bounding the difference between $\wt P_k$ and $P_k$ using standard concentration inequalities. Nonetheless, we would be left with the problem of bounding the $\ell_\infty$ norm of $w_k$ (i.e., the range of the optimistic vector $\wt{h}_k$) over the whole state space, i.e., $\| w_k \|_\infty = sp_{\calS}\{\wt h_k\} = \max_{s\in\calS} \wt h_k(s) - \min_{s\in\calS} \wt h_k(s)$. While in communicating MDPs, it is possible to bound this quantity by the diameter of the MDP as $\SP{h_k} \leq D$~\citep[][Sec. 4.3]{Jaksch10}, in weakly-communicating MDPs $D = +\infty$, thus making this result uninformative. As a result, we need to restrict our attention to the subset of communicating states $\calScom$, where the diameter is finite. We then split the per-step regret over states depending on whether they are explored enough or not as
$\Delta_k \lesssim \sum_{s,a}\nu_k(s,a)(\wt{g}_k -\wt{r}_k(s,a))
\mathbbm{1}\{(s,a) \notin \mathcal{K}_k\} 
+ \rmaxbound \sum_{s,a}\nu_k(s,a)
\mathbbm{1}\lbrace (s,a) \in \mathcal{K}_k\rbrace$.
We start focusing on the poorly visited state-action pairs, \ie $(s,a) \in \mathcal{K}_k$.
In this case \tucrl may suffer the maximum per-step regret $\rmaxbound$ but the number of times this event happen is cumulatively ``small'' (App.~\ref{sec:poorly.visited.states}):
\begin{lemma}\label{lem:sum.bad.states}
 For any $T \geq 1$ and any sequence of states and actions $\lbrace s_1, a_1, \dots \dots s_T, a_T \rbrace$ we have: 
\begin{equation*}\label{eq:badstates}
        \sum_{k=1}^{m}\sum_{s,a}\nu_k(s,a)\mathbbm{1}\lbrace \underbrace{N_k^{\pm}(s,a)\leq\sqrt{\sfrac{t_k}{SA}} }_{(s,a)\in\mathcal{K}_k}\rbrace 
        \leq
        \sum_{t=1}^{T}\mathbbm{1}\left\{ N_{k_t}^{\pm}(s_t,a_t)\leq \sqrt{\sfrac{t}{SA}} \right\} 
        \leq 2\left(\sqrt{\Scom AT} + \Scom A\right)
\end{equation*}
\end{lemma}
When {\color{blue} $(s,a) \notin \mathcal{K}_k$} (\ie $ N_k^\pm(s,a) > \sqrt{\sfrac{t_k}{SA}}$ holds), $\sum_{s,a}\nu_k(s,a)(\wt{g}_k -\wt{r}_k(s,a))\cdot \mathbbm{1}\lbrace (s,a) \notin \mathcal{K}_k \rbrace$ can be bounded as in Eq.~\ref{eq:no.good} but now restricted on $\calScom_k$, so that,
\begin{align*}
        \nu_k(\wt{P}_k -I)\wt{h}_k &= \sum_{s\in \textcolor{blue}{\calScom_k}} \nu_k(s,\wt{\pi}_k(s)) \bigg(\sum_{s'\in \textcolor{blue}{\calScom_k}}\wt{p}_k(s'|s,\wt{\pi}_k(s))w_k(s') - w_k(s) \bigg).
\end{align*}
Since the stopping condition guarantees that $\nu_k(s,\wt{\pi}_k(s)) = 0$ for all $s\in\calStrans_k$, we can first restrict the outer summation to states in $\calScom$. Furthermore, all state-action pairs $(s,a) \notin \mathcal{K}_k$ are such that the optimistic transition probability $\wt p_k(s'|s,a)$ is forced to zero for all $s'\in\calStrans_k$, thus reducing the inner summation. We are then left with providing a bound for the range of $w_k$ \textit{restricted} to the states in $\calScom_k$, i.e., $sp_{\calScom_k}\{w_k\} = \max_{s\in\calScom_k}\{w_k\}$.
We recall that \evi run on a set of plausible MDPs $\wb\M_k$ returns a function $\wt h_k$ such that $\wt{h}_k(s') - \wt{h}_k(s) \leq \rmaxbound \cdot \tau_{\wb{\mathcal{M}}_k}(s \to s')$, for any pair $s,s' \in \calS$, where $\tau_{\wb{\mathcal{M}}_k}(s \to s')$ is the expected shortest path in the extended MDP $\wb\M_k$. Furthermore, since $M^* \in \M_k$,  for all $s, s' \in \calScom_k$, $\tau_{\mathcal{M}_k}(s \to s') \leq \diamcom$. Unfortunately, since $M^*$ may not belong to $\wb\M_k$, the bound on the shortest path in $\M_k$ (i.e., $\tau_{\mathcal{M}_k}(s \to s')$) may not directly translate into a bound for the shortest path in $\wb\M_k$, thus preventing from bounding the range of $\wt h_k$ even on the subset of states in $\calScom_k$. 
Nonetheless, in App.~\ref{app:shortestpath} we show that a minor modification to the confidence intervals of $\wb\M_k$ makes the shortest paths between any two states $s, s' \in \calScom_k$ equivalent in both sets of plausible MDPs, thus providing the bound $sp_{\calScom_k}\{w_k\} \leq \diamcom$.
\footnote{Note that there is not a single way to modify the confidence intervals of $\wb\M_k$ to keep $sp_{\calScom_k}\{w_k\}$ under control.  In App.~\ref{sec:new.relaxation} we present an alternative modifications for which the shortest paths between any two states $s,s' \in \calScom_k$ is not equal but smaller than in $\wb\M_k$ thus ensuring that $sp_{\calScom_k}\{w_k\} \leq \diamcom$.}
The final regret in Thm.~\ref{thm:tucrl.regret} is then obtained by combining all different terms.


\vspace{-.2cm}
\section{Experiments} \label{sec:experiments}
\vspace{-.2cm}
\begin{figure}[t]
        \centering
        \begin{subfigure}[h]{0.4\textwidth}
                \tikz[overlay]{
                        \node [anchor=south west] at (-0.1cm, 2.3cm) {
        \includegraphics[width=\textwidth]{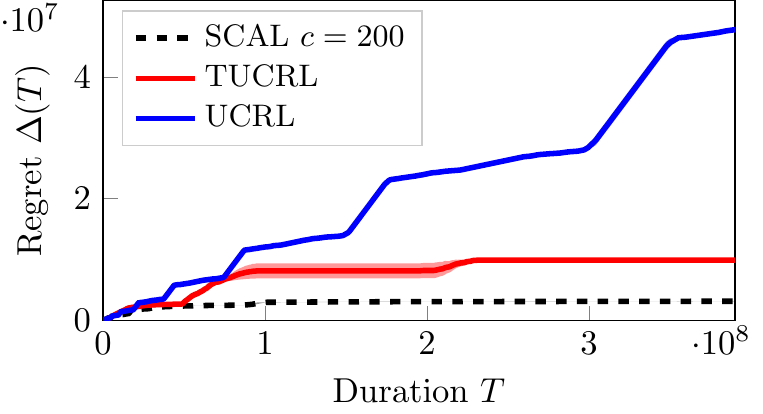}
                        };
                        \node [anchor=south west] at (-0.25cm, -0.5cm) {
	  \includegraphics[width=1.02\textwidth]{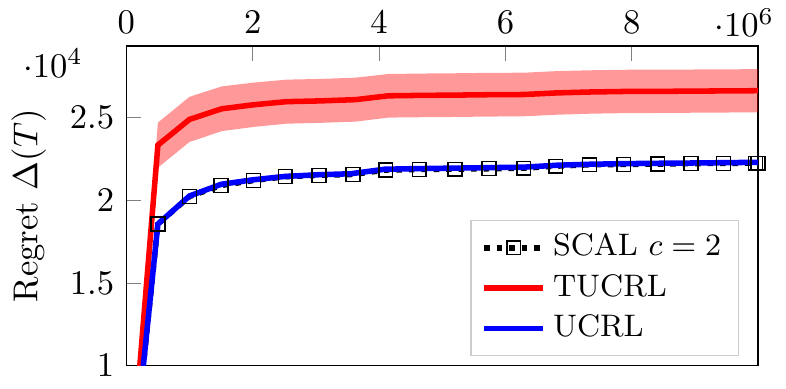}
                        };
                }
        \end{subfigure}
        \hfill
        \begin{subfigure}[b]{0.54\textwidth}
        \includegraphics[width=\textwidth]{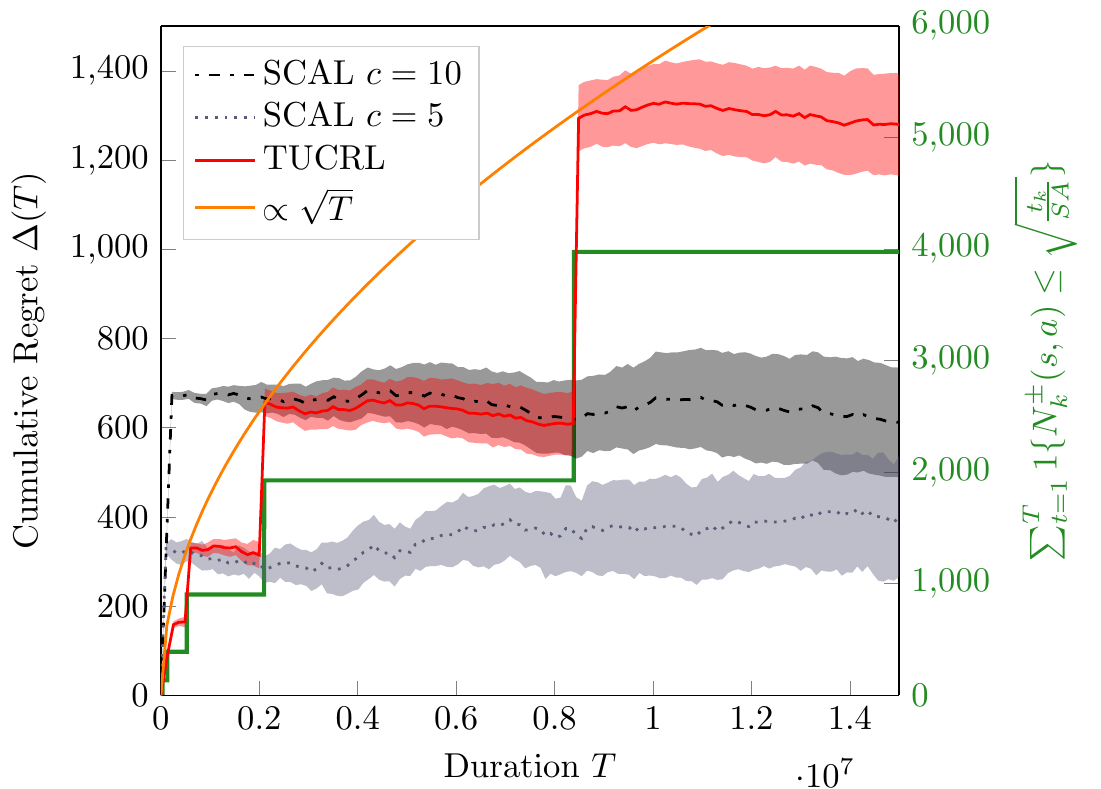}
        \end{subfigure}\hfill
        \vspace{.2cm}
        \caption{Cumulative regret in the taxi with misspecified states  \emph{(left-top)} and in the communicating taxi \emph{(left-bottom)}, and in the weakly communicating three-states domain with $D=+\infty$ \emph{(right)}. Confidence intervals $\beta_{r,k}$ and $\beta_{p,k}$ are shrunk by a factor $0.05$ and $0.01$ for the three-states domain and taxi, respectively. Results are averaged over $20$ runs and 95\% confidence intervals are reported.}
        \label{fig:3states}
\end{figure}

In this section, we present experiments to validate the theoretical findings of Sec.~\ref{sec:tucrl}. We compare \tucrl against \ucrl and \scal.\footnote{To the best of out knowledge, there exists no implementable algorithm to solve the optimization step of \regal and \regald.}
We first consider the taxi problem~\citep{dietterich2000taxi} implemented in OpenAI Gym~\citep{Brockman2016gym}.\footnote{The code is available on \href{https://github.com/RonanFR/UCRL}{GitHub}.}
Even such a simple domain contains \emph{misspecified states}, since the state space is constructed as the outer product of the taxi position, the passenger position and the destination.
This leads to states that cannot be reached from any possible starting configuration (all the starting states belong to $\calScom$).
More precisely, out of $500$ states in $\calS$, $100$ are non-reachable.
On Fig.~\ref{fig:3states}\emph{(left)} we compare the regret of \ucrl, \scal and \tucrl when the misspecified states are present \emph{(top)} and when they are removed \emph{(bottom)}.
In the presence of misspecified states \emph{(top)}, the regret of \ucrl clearly grows linearly with $T$ while \tucrl is able to \emph{learn} as expected.
On the other hand, when the MDP is communicating \emph{(bottom)} \tucrl performs similarly to \ucrl.
The small loss in performance is most likely due to the initial exploration phase during which the confidence intervals on the transition probabilities used by \ucrl (see definition of $\wb{\mathcal{M}}_k$) are tighter than those used by \tucrl (see definition of $\wb{\mathcal{M}}_k^+$). \tucrl uses a ``loose'' bound on the $\ell_1$-norm while \ucrl uses $S$ different bounds, one for every possible next state.
Finally, \scal outperforms \tucrl by exploiting prior knowledge on the bias span. 

We further study \tucrl regret in the simple three-state domain introduced in~\citep{fruit2018constrained} (see App.~\ref{app:experiments} for details) with different reward distributions (uniform instead of Bernouilli).
The environment is composed of only three states ($s_0$, $s_1$ and $s_2$) and one action per state, except in $s_2$ where two actions are available. As a result, the agent only has the choice between two possible policies.
Fig.~\ref{fig:3states}\emph{(left)} shows the cumulative regret achieved by \tucrl and \scal (with different upper-bounds on the bias span) when the diameter is \emph{infinite} \ie $\calScom=\{s_0,s_2\}$ and $\calStrans =\{s_1\}$ (we omit \ucrl, since it suffers linear regret).
Both \scal and \tucrl quickly achieve sub-linear regret as predicted by theory.
However, \scal and \tucrl seem to achieve different growth rates in regret: while \scal appears to reach a logarithmic growth, the regret of \tucrl seems to grow as $\sqrt{T}$ with \emph{periodic} ``jumps'' that are increasingly distant (in time) from each other.
This can be explained by the way the algorithm works: while most of the time \tucrl is optimistic on the restricted state space $\calScom$ (\ie $\calScom_k = \calScom$), it \emph{periodically} allows transitions to the set $\calStrans$ (\ie $\calScom_k = \calS$), which is indeed not reachable.
Enabling these transitions triggers aggressive \emph{exploration} during an entire episode. The policy played is then sub-optimal creating a ``jump'' in the regret. At the end of this \emph{exploratory episode}, $\calScom_k$ will be set again to $\calScom$ and the regret will stop increasing until the condition $N_k^\pm \leq \sqrt{\sfrac{t_k}{SA}}$ occurs again (the time between two consecutive exploratory episodes grows quadratically).
The cumulative regret incurred during exploratory episodes can be bounded by the term plotted in green on Fig.~\ref{fig:3states}\emph{(left)}. In Lem.~\ref{lem:sum.bad.states} we proved that this term is always bounded by $O(\sqrt{\Scom AT})$. Therefore, it is not surprising to observe a $\sqrt{T}$ increase of both the green and red curves.
Unfortunately, the growth rate of the regret will keep increasing as $\sqrt{T}$ and will never become logarithmic unlike \scal (or \ucrl when the MDP is communicating). This is because the condition $N_k^\pm \leq \sqrt{\sfrac{t_k}{SA}}$ will always be triggered $\Theta(\sqrt{T})$ times for any $T$. In Sec.~\ref{sec:no_free_lunch} we show that this is not just a drawback specific to \tucrl, but it is rather an \emph{intrinsic limitation} of learning in weakly-communicating MDPs.


\vspace{-.2cm}
\section{Exploration-exploitation dilemma with infinite diameter}\label{sec:no_free_lunch}
\vspace{-.2cm}

\begin{figure}[t]
\begin{subfigure}[h]{0.57\textwidth}
 \begin{tikzpicture}[domain=0:4, font=\small, scale=0.65]
    \draw[very thin,color=gray] (-0.1,-0.1) (10,3.9);
    \draw[->] (-0.2,0) -- (10,0) node[right] {$T$};
    \draw[->] (0,-0.2) -- (0,5.) node[above, xshift=1cm] { $\mathbb{E}[\Delta(\textsc{UCRL},T,M)]$};
    
    \draw[scale=1,domain=0:5,smooth,variable=\x,ForestGreen,samples=100,line width=0.7pt] plot ({\x},{0.8*1.2*\x});
    \draw[scale=1,domain=0:10,smooth,variable=\x,blue,samples=400,line width=0.7pt] plot ({\x},{1.3*1.2*\x^(0.5)});
    \draw[scale=1,domain=0.03:10,smooth,variable=\x,red,samples=100,line width=0.7pt] plot ({\x},{1.2*2.1+0.65*1.2*ln(\x)});
    
    \draw[scale=1,domain=0:2.64,smooth,variable=\x,black,samples=100, line width=1.75pt] plot ({\x},{0.8*1.2*\x});
    \draw[scale=1,domain=2.64:6.49,smooth,variable=\x,black,samples=100, line width=1.75pt] plot ({\x},{1.3*1.2*\x^(0.5)});
    \draw[scale=1,domain=6.49:10,smooth,variable=\x,black,samples=100, line width=1.75pt] plot ({\x},{1.2*2.1+1.2*0.65*ln(\x)});
    
    \node () at (5.2,5.28){\color{ForestGreen}$O(T)$};
    \node () at (8.4,5.28){\color{blue}$O(DS\sqrt{AT\ln(T)})$};
    \node () at (1.9,3.96){\color{red}$O\left(\frac{D^2 S^2 A}{\gamma} \ln(T)\right)$};
    
    \node (1) at (2.64,2.604){};
    \node (2) at (2.64,-0.1){};
    \node (3) at (6.49,4.08){};
    \node (4) at (6.49,-0.1){};
    
    \node () at (-0.2,-0.4){$0$};
    \node () at (2.64,-0.3){${T}_M^\dagger$};
    \node () at (6.49,-0.3){${T}_M^*$};
    
    \node (5) at (8.5,2.4){Regret upper-bound};
    \node (6) at (8.5,4.32){};
    
    \draw[-,dashed, >=latex](1) to node[above]{} (2);
    \draw[-,dashed, >=latex](3) to node[above]{} (4);
    
    \draw[-,->] (5) to (6);
    
    \node[anchor=west] at (-1.5, 6) {\parbox{0.1\textwidth}{\subcaption{\label{subfig:hi}}}};
\end{tikzpicture}
\end{subfigure}
\begin{subfigure}[h]{0.42\textwidth}
\centering
    \vspace{-.5cm}
  \begin{tikzpicture}
          \begin{scope}[local bounding box=scope1,scale=0.63]
	\tikzset{VertexStyle/.style = {draw, 
									shape          = circle,
	                                text           = black,
	                                inner sep      = 2pt,
	                                outer sep      = 0pt,
	                                minimum size   = 24 pt}}
	\tikzset{VertexStyle2/.style = {shape          = circle,
	                                text           = black,
	                                inner sep      = 2pt,
	                                outer sep      = 0pt,
	                                minimum size   = 14 pt}}
	\tikzset{Action/.style = {draw, 
                					shape          = circle,
	                                text           = black,
	                                fill           = black,
	                                inner sep      = 2pt,
	                                outer sep      = 0pt}}
	                                 

	\node[VertexStyle,fill=gray](s0) at (0,0) {$ x $};
	\node[Action](a0s0) at (.5,1.25){};
	\node[Action](a1s0) at (.5,-1.25){};
	\node[VertexStyle](s2) at (4,0){$y$};
	\node[Action](a0s2) at (2.6,0){};
	\node[Action](a1s2) at (5.4,0.){};
    
	\draw[->, >=latex, double, color=red](s0) to node[midway, right]{{ $b$}} (a0s0);
	\draw[->, >=latex, double, color=red](s0) to node[midway, right]{{ $d$}} (a1s0);
	\draw[->, >=latex](a0s0) to [out=45,in=120,looseness=0.8] node[above]{$\varepsilon$} (s2);
	\draw[->, >=latex](a0s0) to [out=80,in=140,looseness=2.1] node[above]{$1-\varepsilon$} (s0);
	\draw[->, >=latex](a1s0) to [out=-150,in=-140,looseness=2.1] node[above]{} (s0);
	\draw[->, >=latex, double, color=red](s2) to node[midway, above]{{ $d$}} (a0s2);   
	\draw[->, >=latex](a0s2) to [out=180, in=360, looseness=0.] node[above]{} (s0);
	\draw[->, >=latex, double, color=red](s2) to node[midway, below]{{ $b$}} (a1s2);
	\draw[->, >=latex](a1s2) to [out=-10, in=-85, looseness=2.1] (s2);

    \node [right, xshift=0.2em, yshift=0.2em] at (a0s0) { $r=0$};
    \node [right, xshift=-6em, yshift=-0.2em] at (a1s0) { $r=\sfrac{1}{2}$};
    \node [below, yshift=-0.2em] at (a0s2) { $r = 0$};
    \node [above, yshift=0.2em] at (a1s2) { $r =1$};  
    
        \node[anchor=west] at ($(s2)+(0.8,1.4cm)$) {\parbox{.5cm}{\subcaption{\label{fig:delta_pos}}}};
    \end{scope}
    \begin{scope}[shift={($(scope1.south)+(-1.2cm,-0.9cm)$)}, scale=0.7]
	\tikzset{VertexStyle/.style = {draw, 
									shape          = circle,
	                                text           = black,
	                                inner sep      = 2pt,
	                                outer sep      = 0pt,
	                                minimum size   = 24 pt}}
	\tikzset{VertexStyle2/.style = {shape          = circle,
	                                text           = black,
	                                inner sep      = 2pt,
	                                outer sep      = 0pt,
	                                minimum size   = 14 pt}}
	\tikzset{Action/.style = {draw, 
                					shape          = circle,
	                                text           = black,
	                                fill           = black,
	                                inner sep      = 2pt,
	                                outer sep      = 0pt}}
	                                 
	\node[VertexStyle,fill=gray](s0) at (0,0) {$ x $};
	\node[Action](a0s0) at (.4,1.2){};
	\node[Action](a1s0) at (.4,-1.2){};
	\node[VertexStyle](s2) at (4,0){$y$};
	\node[Action](a0s2) at (2.9,-0.5){};
	\node[Action](a1s2) at (2.9,0.5){};
    
	\draw[->, >=latex, double, color=red](s0) to node[midway, right]{{ $b$}} (a0s0);
	\draw[->, >=latex, double, color=red](s0) to node[midway, right]{{ $d$}} (a1s0);
	\draw[->, >=latex](a0s0) to [out=150,in=140,looseness=2.1] node[above]{} (s0);
	\draw[->, >=latex](a1s0) to [out=-150,in=-140,looseness=2.1] node[above]{} (s0);
	\draw[->, >=latex, double, color=red](s2) to node[midway, below]{{ $b$}} (a0s2);   
	\draw[->, >=latex](a0s2) to [out=200, in=-20, looseness=0.8] node[above]{} (s0);
	\draw[->, >=latex, double, color=red](s2) to node[midway, above]{{ $d$}} (a1s2);
	\draw[->, >=latex](a1s2) to [out=160, in=20, looseness=0.8] (s0);

    \node [right, xshift=-5.5em, yshift=-0.4em] at (a0s0) { $r=0$};
    \node [right, xshift=-6em, yshift=-0.2em] at (a1s0) { $r=\sfrac{1}{2}$};
    \node [above, xshift=-0.5em,yshift=0.2em] at (a0s2) { $r = 1$};
    \node [above, xshift=-0.5em,yshift=0.2em] at (a1s2) { $r =0$};  
    
        \node[anchor=west] at ($(s2)+(-0,-1.4cm)$) {\parbox{0.2\textwidth}{\subcaption{\label{fig:delta_zero}}}};
     \end{scope}
    \end{tikzpicture}
\end{subfigure}
\vspace{-.2cm}
\caption{\ref{subfig:hi} Expected regret of $\ucrl$ (with known horizon $T$ given as input) as a function of $T$. \ref{fig:delta_pos}~\ref{fig:delta_zero} Toy example illustrating the difficulty of learning non-communicating MDPs. We represent a family of possible MDPs $\mathcal{M} = (M_\varepsilon)_{\varepsilon \in [0,1]}$ where the probability $\varepsilon$ to go from $x$ to $y$ lies in $[0,1]$. }
\label{fig:ucrl.regimes_example.misspec}
\vspace{-.3cm}
\end{figure}
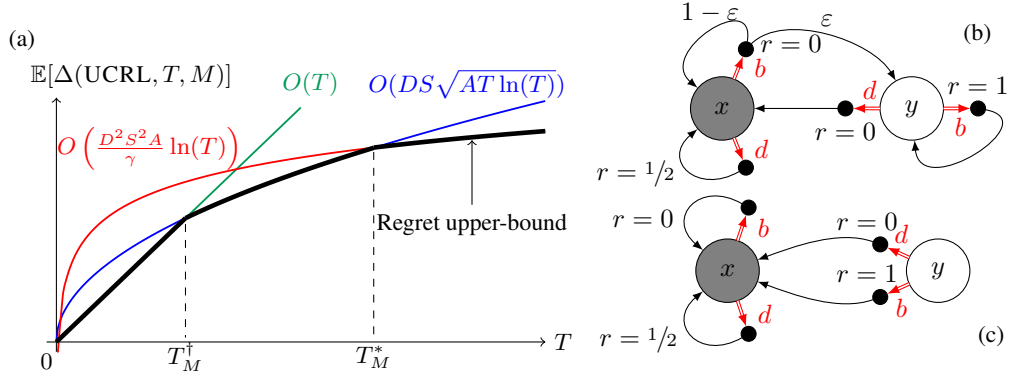

In this section we further investigate the empirical difference between \scal and \tucrl and prove an impossibility result \emph{characterising} the \emph{exploration-exploitation dilemma} when the diameter is allowed to be \emph{infinite} and \emph{no prior knowledge} on the optimal bias span is available.

We first recall that the expected regret $\mathbbm{E}[\Delta(\ucrl,M,T)]$ of $\ucrl$ (with input parameter $\delta =1/3T$) after $T\geq 1$ time steps and for any finite MDP $M$ can be bounded in several ways: 
\begin{align}\label{eqn:regret_bound_ucrl}
 \hspace{-0.1in}\mathbbm{E}[\Delta(\ucrl,M,T)] \leq \begin{cases}
                                    \rmaxbound T~~\text{(by definition)}\\
                                    C_1 \cdot\rmaxbound D \sqrt{\Gamma SAT\ln(3T^2)} +\frac{1}{3} ~\text{\citep[Theorem 2]{Jaksch10}}\\
                                    C_2 \cdot \rmaxbound \frac{D^2 \Gamma S A}{\gamma} \ln(T) + C_3(M)~\text{\citep[Theorem 4]{Jaksch10}}
                                   \end{cases}
\end{align}
where $\gamma=g^*_M - \max_{s,\pi}\{g^\pi_M(s):~ g^\pi_M(s) < g^*_M \} $ is the gap in gain, $C_1:= 34$ and $C_2:=34^2$ are numerical constants independent of $M$, and $C_3(M):= O(\max_{\pi: \pi(s)=a} T_{\pi})$
with $T_\pi$ a measure of the ``mixing time'' of policy $\pi$. The three different bounds lead to three different \emph{growth rates} for the function $T \longmapsto \mathbbm{E}[\Delta(\ucrl,M,T)]$ (see Fig.~\ref{subfig:hi}):
 1) for ${T}_M^\dagger \geq T \geq 0$, the expected regret is linear in $T$,
 2) for ${T}_M^* \geq T \geq {T}_M^\dagger$ the expected regret grows as $\sqrt{T}$,
 3) finally for $T\geq {T}_M^*$, the increase in regret is only logarithmic in $T$.
These different \emph{``regimes''} can be observed empirically (see~\citep[Fig. 5, 12]{fruit2018constrained}).
Using ~\eqref{eqn:regret_bound_ucrl}, it is easy to show that the time it takes for \ucrl to achieve sub-linear regret is at most ${T}_M^\dagger = \wt{O}(D^2 \Gamma S A)$. We say that an algorithm is \textit{efficient} when it achieves sublinear regret after a number of steps that is polynomial in the parameters of the MDP (i.e., \ucrl is then \emph{efficient}).
We now show with an example that \emph{without prior knowledge}, any \emph{efficient} learning algorithm must satisfy ${T}_M^* =+\infty$ when $M$ has \emph{infinite diameter} (i.e., it cannot achieve logarithmic regret).


\begin{example}\label{ex:weakly_communicating}
        We consider a family of weakly-communicating MDPs $\mathcal{M} = (M_\varepsilon)_{\varepsilon \in [0,1]}$ represented on Fig.~\ref{fig:ucrl.regimes_example.misspec}\emph{(right)}.
 Every MDP instance in $\mathcal{M}$ is characterised by a specific value of $\varepsilon \in [0,1]$ which corresponds to the probability to go from $x$ to $y$. For $\varepsilon>0$ (Fig.~\ref{fig:delta_pos}), the optimal policy of $M_\varepsilon$ is such that $\pi^*(x) = b$ and the optimal gain is $g^*_\varepsilon = 1$ while for $\varepsilon=0$ (Fig.~\ref{fig:delta_zero}) the optimal policy is such that $\pi^*(x) = d$ and the optimal gain is $g^*_0 = 1/2$. We assume that the learning agent knows that the true MDP $M^*$ belongs to $\mathcal{M}$ but does not know the value $\varepsilon^*$ associated to $M^*=M_{\varepsilon^*}$. We assume that all rewards are deterministic and that the agent starts in state $x$ (coloured in grey).
\end{example}

\begin{lemma}\label{lem:weakly_communicating}
 Let $C_1, C_2,\alpha, \beta >0$ be positive real numbers and $f$ a function defined for all $\varepsilon \in ]0,1]$ by $f(\varepsilon) =  C_1 (1/\varepsilon)^{\alpha}$.
  There exists no learning algorithm $\mathfrak{A}_T$ (with known horizon $T$) satisfying both
  \vspace{-0.07in}
 \begin{enumerate}[leftmargin=*,nolistsep, topsep=0pt]
  \item for all $\varepsilon \in ]0,1]$, there exists ${T}_\varepsilon^\dagger \leq f(\varepsilon)$ such that $\mathbbm{E}[\Delta(\mathfrak{A}_T,M_\varepsilon,x,T)] < 1/6\cdot T$ for all $T \geq {T}_\varepsilon^\dagger$,
  \item and there exists ${T}_0^* < +\infty$ such that $\mathbbm{E}[\Delta(\mathfrak{A}_T,M_0,x,T)] \leq C_2 (\ln(T))^\beta$ for all $T \geq {T}_0^*$.
  \end{enumerate}
\end{lemma}

Note that point \emph{1} in Lem.~\ref{lem:weakly_communicating} formalizes the concept of \emph{``efficient learnability''} introduced by \citet[Section 11.6]{SuttonB17} \ie ``learnable within a polynomial rather than exponential number of time steps''.
All the MDPs in $\mathcal{M}$ share the same number of states $S=2 \geq \Gamma$, number of actions $A=2$, and gap in average reward $\gamma=1/2$. As a result, any function of $S$, $\Gamma$, $A$ and $\gamma$ will be considered as constant. For $\varepsilon >0$, the diameter coincides with the optimal bias span of the MDP and $D=\SP{h^*}=1/\varepsilon <+\infty$, while for $\varepsilon = 0$, $D = +\infty$ but $\SP{h^*} = 1/2 $. 
As shown in Eq.~\ref{eqn:regret_bound_ucrl} and Thm.~\ref{thm:tucrl.regret}, \ucrl and \tucrl satisfy property \emph{1.} of Lem.~\ref{lem:weakly_communicating} with $\alpha =2$ and $C_1 = O(S^2 A)$ but do not satisfy \emph{2.} On the other hand, \scal satisfies \emph{2.} with $\beta =1$ and $C_2 = O(H^2 SA/\gamma)$ (although this result is not available in the literature, it is straightforward to adapt the proof of \ucrl \citep[Theorem 4]{Jaksch10} to \scal)
but since \citep[Theorem 12]{fruit2018constrained} holds only when $H\geq \SP{h^*}$, \scal only satisfies \emph{1.} for $\varepsilon \geq 1/H$ and $\varepsilon=0$ (not for $\varepsilon \in ]0, 1/H[$).
Lem.~\ref{lem:weakly_communicating} proves that no algorithm can actually achieve both \emph{1.} and \emph{2.} As a result, since \tucrl satisfies \emph{1.}, it cannot satisfy \emph{2.} 
This matches the empirical results presented in Sec.~\ref{sec:experiments} where we observed that when the diameter is infinite, the growth rates of the regret of \scal and \tucrl were respectively logarithmic and of order $\Theta(\sqrt{T})$. 
An algorithm that does not satisfy \emph{1.} could potentially satisfy \emph{2.} but, by definition of \emph{1.}, it would suffer linear regret for a number of steps that is more than \emph{polynomial} in the parameters of the MDP (more precisely, $e^{D^{1/\beta}}$). This is not a very desirable property and we claim that an \emph{efficient} learning algorithm should always prefer \emph{finite time guarantees} (\emph{1.}) over \emph{asymptotic guarantees} (\emph{2.}) when they cannot be accommodated.


\vspace{-.2cm}
\section{Conclusion}\label{sec:conclusion}
\vspace{-.2cm}
We introduced \tucrl, an algorithm that efficiently balances exploration and exploitation in weakly-communicating and multi-chain MDPs, when the starting state $s_1$ belongs to a communicating set (Asm.~\ref{asm:initial.state}).
We showed that \tucrl achieves a square-root regret bound
and that, in the general case, it is not possible to design algorithm with logarithmic regret and polynomial dependence on the MDP parameters.
Several questions remain open: \textbf{1)} relaxing Asm.~\ref{asm:initial.state} by considering a transient initial state (i.e., $s_1\in\calStrans$), \textbf{2)} refining the lower bound of~\citet{Jaksch10} to finally understand whether it is possible to scale with $\SP{h^*}$ (at least in communicating MDPs) instead of $D$ without any prior knowledge (the flaw in \regald may suggest it is indeed impossible). 

%

\subsubsection*{Acknowledgments}
This research was supported in part by French Ministry of Higher Education and Research, Nord-Pas-de-Calais Regional Council and French National Research Agency (ANR) under project ExTra-Learn (n.ANR-14-CE24-0010-01).

\bibliography{tucrl}
\bibliographystyle{unsrtnat}

\clearpage
\appendix

\section{Mistake in the regret bound of \regald}\label{app:mistake_regal}

\subsection{Regularized optimistic RL (\regal)}\label{sec:regal}
In weakly communicating MDPs, to avoid the over-optimism of \ucrl, \citet{Bartlett2009regal} proposed to penalise the optimism on $g^*$ by the optimal bias span $\SP{h^*}$.
Formally, at each episode $k$, their algorithm --\regal-- solves the following optimization problem:
\begin{equation}\label{eq:regal}
        \wt{M}_k = \argmax_{M \in \mathcal{M}_k} \{g^*_{M} - C_k \cdot \SP{h^*_M}\}
\end{equation}
where $C_k \geq 0$ is a regularisation coefficient.
Note that such optimization requires to first compute the optimal policy for a given MDP $M \in \mathcal{M}_k$ and then evaluate the regularized gain.
Implicitly, this defines the optimistic policy $\wt{\pi}_k = \argmax_{\pi \in \Pi^{\textsc{SD}}} \{g^\pi_{\wt{M}_k}\}$.
The term $\SP{h^*}$ can be interpreted as a measure of the \emph{complexity} of the environment: the bigger $\SP{h^*}$, the more difficult it is to achieve the stationary reward $g^*$ by following the optimal policy.
In \emph{supervised learning}, regularisation is often used to penalise the objective function by a measure of the complexity of the model so as to avoid \emph{overfitting}.
It is thus reasonable to expect that \emph{over-optimism} in online RL can also be avoided through regularisation.

The regret bound of \regal holds only when $C_k$ is set to $\Theta (1/\sum_{s,a} \nu_k(s,a))$. 
This means that \regal requires the knowledge of (future) visit counts $\nu_k(s,a)$ before episode $k$ begins in order to tune the regularisation coefficient $C_k$.
Unfortunately, an episode stops when the number of visits in a state-action pair $(s,a) \in \calS \times \mathcal{A}$ has doubled and it is not possible to predict the future sequence of states of a given policy for two reasons: 1) the true MDP $M^*$ is unknown and 2) what is observed is a random \emph{sampled} trajectory (as opposed to \emph{expected}). As a result, \regal is not \emph{implementable}.
\citet{Bartlett2009regal} proposed an alternative algorithm --\regald-- that leverages on the \emph{doubling trick} to guess the length of episode $k$ (\ie $\sum_{s,a} \nu_k(s,a)$) and proved a slightly worse regret bound than for \regal.
\regald divides an episode $k$ into sub-iterations where it applies the doubling trick techniques.
At each sub-iteration $j$, \regald guesses that the length of the episode will be at most $2^j$ and it solves problem~\eqref{eq:regal} with
$C_{k,j} \propto 1/\sqrt{2^j}$.
Then, it executes the optimistic policy $\wt{\pi}_{k,j}$ on the true MDP until the \ucrl stopping condition is reached or $2^j$ steps are performed.
In the first case the episode $k$ ends since the guess was correct, while, in the second case, a new sub-iteration $j+1$ is started.
This implies that for any $k,j$: 
\begin{equation}\label{eq:regal.stopping.guarantees}
        \sum_{s,a} \nu_{k,j}(s,a) \leq 2^j,
\end{equation}
where $\nu_{k,j}(s,a)$ denotes the number of visits to $(s,a)$ during episode $k$ and sub-iteration $j$.

\subsection{The doubling trick issue}

The mistake in \regald is located in the proof of the regret ~\citep[Theorem 3]{Bartlett2009regal} (see Sec. 6.3).
Let $\wt{h}_{k,j}$ denote the optimistic bias span at episode $k$ and sub-iteration $j$ induced by the doubling trick.
At a high level, the mistake comes from the attempt to upper-bound the term $x \cdot \sum_{s,a} \nu_{k,j}(s,a) $ by $x \cdot 2^j $ (for a given $x$) using the fact the $\sum_{s,a} \nu_{k,j}(s,a) \leq 2^j$.
Unfortunately, this is possible only under the assumption that $x \geq 0$ that does not hold in the case of \regald.

Formally, while bounding $\sum_{k \in G} \Delta_k$, the authors have to deal with the term (derived by the combination of~\citep[Eq. 15]{Bartlett2009regal} and~\citep[Lem. 11]{Bartlett2009regal} with~\citep[Eq. 14]{Bartlett2009regal}):
\[
        U := \sum_{k\in G}\sum_{j} \SP{\wt{h}_{k,j}} \left( c \sqrt{\sum_{s,a} \nu_{k,j}(s,a)} - C_{k,j} \sum_{s,a}\nu_{k,j}(s,a) \right)
\]
where $c :=  2 S \sqrt{12 \ln(2AT/\delta)} + \sqrt{2  \ln(1/\delta)} \geq 0$ and recall that $\sum_{s,a}\nu_{k,j}(s,a)$ denotes the \emph{actual} length of the episode $k$ at sub-iteration $j$.
In the \regald proof the authors directly replaced the actual length of the episode with the guessed length $2^j := \ell_{k,j}$ showing that the first term can be upper-bounded by $c \cdot \sqrt{\sum_{s,a} \nu_{k,j}(s,a)} \leq c \cdot \sqrt{2^j}$ (due to Eq.~\ref{eq:regal.stopping.guarantees}). 
Concerning the second term, they write $- C_{k,j} \sum_{s,a}\nu_{k,j}(s,a) {\color{red}\bm{\leq}} -C_{k,j} 2^j$. Since $-C_{k,j} := -c / \sqrt{2^j} \leq 0$ is negative, this last inequality is not true and the reverse inequality holds instead (using Eq.~\ref{eq:regal.stopping.guarantees}): $- C_{k,j} \sum_{s,a}\nu_{k,j}(s,a) {\color{blue}\bm{\geq}} -C_{k,j} 2^j$.
Therefore, it is not possible to guarantee that $U \leq 0$ as claimed by~\citet{Bartlett2009regal} (the authors probably didn’t pay attention to the sign). 
To do this, we would need to \emph{lower-bound} $\sum_{s,a}\nu_{k,j}(s,a)$. Unfortunately, the only lower bound with probability $1$ available for that term is $\min_{s,a} \{N_k(s,a)\} +2$. 
This is not big enough to cancel the term $ c \sqrt{\sum_{s,a}\nu_{k,j}(s,a)}$ 
and $C_{k,j}$ needs to be increased. 
As a result, the term $\SP{h^\star} \sum_{k \in G} \sum_{j} C_{k,j} \sqrt{\sum_{s,a}\nu_{k,j}(s,a)}$ becomes too big and all the proof collapses.

Notice that a similar mistake is contained in the work by~\citet{pmlr-v28-maillard13} where they use a regularized approach to learn a state representation in online settings.
Similarly to~\citep{Bartlett2009regal},  the authors have to bound the term $\sum_{s,a} \nu_{k,j}(s,a) (g^* - \wt{g}_{k,j})$.
By exploiting the fact that $g^* - \wt{g}_{k,j} \leq \alpha$ (we omit the full expression of $\alpha$ for sake of clarity)~\citep[Eq. 17 Sec. 5.2]{pmlr-v28-maillard13} the authors derived the bound $\sum_{s,a} \nu_{k,j}(s,a) (g^* - \wt{g}_{k,j}) \leq 2^j \cdot \alpha$~\citep[Eq. 18]{pmlr-v28-maillard13}.
The difference $g^* - \wt{g}_{k,j}$ might be negative in which case the result does not hold.
Actually for the case in which there is no regularization $C_{k,j}=0$, $g^* \leq \wt{g}_{k,j}$ which is what is used in the regret proof of \ucrl. Therefore, it is very likely that the sign of $g^* - \wt{g}_{k,j}$ can sometimes be negative.

In conclusion, it seems unavoidable to use a \emph{lower-bound} (and not an upper-bound) on $\sum_{s,a}\nu_{k,j}(s,a)$ to derive a correct regret bound for \regald. As already mentioned, given the current stopping condition of an episode, the only reasonable lower bound is $\min_{s,a} \{N_k(s,a)\} +2$ and it does not seem sufficient to derive a sensible regret bound. Another research direction could be to change the stopping condition. However, one of the terms in the regret bound of \regal (and of \regald) scales as $m \sqrt{T} \log_2(T)$ where $m$ is the number of episodes. The term $m$ is highly sensitive to the stopping condition and there is very little margin if we want to avoid $m \sqrt{T} \log_2(T)$ to become the leading term in the regret bound. All the efforts we put in this direction were unsuccessful. We conjecture that regularising by the optimal bias span might not allow to learn MDPs with infinite diameter.

\section{Unbounded optimal bias span with continuous Bayesian priors/posteriors}\label{app:unbounded.span.bayesian}
Recently, \citet{NIPS2017_6732} and \citet{theocharous2017largeps} proposed posterior sampling algorithms and proved bounds on the expected Bayesian regret. The regret bounds that they derive scale linearly with $H$, where $H$ is the highest optimal bias span of all the MDPs that can be drawn from the prior/posterior distribution.
Formally, let $f(\theta)$ be the density function of the prior/posterior distribution over the family of MDPs $(M_\theta)$ parametrised by $\theta$. Then:
\[
        H := \sup_{\theta: f(\theta) > 0} \{\SP{h^*_{\theta}}\}.
\]
In this section we present an example where $H$ is \emph{infinite} and argue that it is probably the case for most priors/posteriors used in practice.

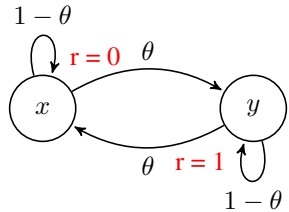
\begin{figure}[b]
  \centering
  \begin{tikzpicture}[->, >=stealth', shorten >=1pt, auto, node distance=2.8cm, semithick]
        \node[state] (A) {$x$};
        \node[state] (B) [right of=A] {$y$};
        \node (C) at (0.7,0.7) {{\color{red} r = 0}};
        \node (C) at (2.1,-0.7) {{\color{red} r = 1}};
        
        \path (A) edge[bend left] node {$\theta$} (B);
        \path (A) edge[loop above] node {$1-\theta$} (A);

        \path (B) edge[bend left] node {$\theta$} (A);
        \path (B) edge[loop below] node {$1-\theta$} (B);
  \end{tikzpicture}
\caption{\label{F:toy_example}
        Toy example of a parametrised MDP $M_\theta$ with a single policy (one action per state).
}
\end{figure}

\begin{example}[Unbounded optimal bias span with continuous prior/posterior]\label{ex:unbounded_span}
 
 Consider the example of Fig.~\ref{F:toy_example}. There is only one action in every state and so one optimal policy. The (unique) action that can be played in state  $s_0$ loops on $s_0$ with probability $1-\theta$ and goes to $s_1$ with probability $\theta$. The reward associated to this action is $0$. Symmetrically, the (unique) action that can be played in state  $s_1$ loops on $s_1$ with probability $1-\theta$ and goes to $s_1$ with probability $\theta$. The reward associated to this action is $1$.
 This MDP is characterised by the parameter $\theta$ and we denote it by $M_\theta$. For any $\theta \in [0,1]$, we denote by $g^*_\theta$ (resp. $h^*_\theta$) the optimal gain (resp. bias) of $M_\theta$. Observe that when $\theta >0$, $M_\theta$ is ergodic and therefore the optimal gain $g^*_\theta = 1/2$ is state-independent whereas when $\theta = 0$, $M_\theta$ is multichain and the optimal gain does depend on the initial state: $g^*_0(x) =0 < 1 = g^*_0(y) $.
\end{example}
 
 Let's assume that the prior/posterior distribution we use on $M_\theta$ is characterised by a probability density function $f$ satisfying $f(\theta) >0$ for all $\theta > 0$ and $f(0) = 0$. Note that this assumption does not constrain the ``smoothness'' of $f$ e.g., $f$ can have continuous derivatives of all orders. Under this assumption, $f$ is non-zero only for ergodic MDPs. It goes without saying that for all $\theta \in [0,1]$ ($0$ included), $\SP{h^*_\theta} < +\infty$ by definition (the optimal bias span is always finite). More precisely we have:
 \begin{align*}
   g^*_\theta = \begin{cases}
           [1/2,1/2]^\transp & \text{if} ~~ \theta >0\\
           [0,1]^\transp & \text{if} ~~ \theta =0
                     \end{cases}
                       \quad~~~ \text{and} ~~~\quad
  \SP{h^*_\theta} = \begin{cases}
          \frac{1}{2\theta} & \text{if} ~~ \theta >0\\
          0 & \text{if} ~~ \theta =0
                     \end{cases}
 \end{align*}
  As a result, although $\SP{h^*_\theta}$ is always \emph{finite}, \ie  $\forall \theta \in [0,1],~\SP{h^*_\theta} < +\infty$, it is \emph{unbounded} on the set of plausible MDPs $\theta \in ]0,1]$ satisfying $f(\theta) >0$, \ie 
  \[
          H:=\sup_{\theta \in ]0,1]}\lbrace \SP{h^*_\theta}\rbrace = \lim_{\theta \to 0^+} \frac{1}{2\theta}= +\infty
  \]
Therefore, the regret bound $\wt{O}(HS\sqrt{AT})$ proved by \citet{NIPS2017_6732, theocharous2017largeps} does not hold with prior/posterior $f$ since $H = +\infty$.
One might argue that the proofs in \citep{NIPS2017_6732, theocharous2017largeps} could be fixed by showing that $H$ is bounded with probability 1.
Unfortunately, for any $C \in [0,+\infty[$, the probability $\mathbbm{P}(\SP{h^*_\theta} \geq C)=\int_{\theta =0}^{\frac{1}{2C}} f(\theta) d\theta >0$  of sampling an MDP with $\SP{h^*_\theta} \geq C$ is strictly positive.
We therefore conjecture that for this specific choice of priors/posteriors, the regret proof in \citep{NIPS2017_6732, theocharous2017largeps} cannot be fixed without major changes and new arguments.
More generally, let's imagine that we have a prior/posterior distribution $f$ satisfying:
\begin{itemize}[topsep=0mm]
\item there exists $\theta_0$ such that $M_{\theta_0}$ has non-constant gain \ie $\SP{g_{\theta_0}^*} >0$,
\item there exists an open neighbourhood of $\theta_0$ denoted $\Theta_0$ such that $\forall \theta \in \Theta_0$, $M_\theta$ has constant gain (\eg $M_\theta$ is weakly-communicating) and $f(\theta) > 0$.
\end{itemize}
  In this case we will face the same problem as in Ex.~\ref{ex:unbounded_span} i.e., 
\begin{align*}
 \sup_{\theta:~f(\theta) >0 }\lbrace \SP{h^*_\theta}\rbrace = +\infty ~~ \text{and}~~\forall C\in [0,+\infty[,~\mathbbm{P}(\SP{h^*_\theta} \geq C) >0
\end{align*}
When the set of plausible MDPs is \emph{finite}, this problem cannot occur. But most priors/posteriors used in practice are \emph{continuous} distributions. For instance, a Dirichlet distribution will most likely satisfy the above assumptions.

%


\section{Algorithmic Details}\label{app:algorithm}

For technical reasons (see App.~\ref{app:shortestpath}), we consider a slight \emph{relaxation of the optimization problem}~\eqref{eq:tucrl.optimization} in which $\wb{\mathcal{M}}_k$ is replaced by a relaxed extended MDP $\wb{\mathcal{M}}_k^+ \supseteq \wb{\mathcal{M}}_k$ defined by using $\ell_1$-norm concentration inequalities for $p(\cdot|s,a)$.\footnote{We recently noticed that is possible to obtain a tighter relaxation that preserves the Bernstein nature of the confidence intervals (instead of resorting to $\ell_1$-norm). This version may be more efficient in practical applications. More details on this are reported in Sec.~\ref{sec:new.relaxation}. }
Let $B_{p,k}^+(s,a) = \{\wt{p}(\cdot|s,a)\in\mathcal{C} :~\|\wt{p}(\cdot|s,a) - \wh{p}(\cdot|s,a)\|_1 \leq \sum_{s'} \beta_{p,k}^{sas'}\}$ (resp. $\wb{B}_{p,k}^+$) be the relaxed confidence interval, then $\mathcal{M}_k^+$ (resp. $\wb{\mathcal{M}}_k^+$) is the corresponding (relaxed) set of plausible MDPs.
This relaxed optimistic optimization problem is solved by running extended value iteration (EVI) on $\wb{\mathcal{M}}_k^+$ (up to accuracy $\epsilon_k = \rmaxbound/\sqrt{t_k}$).
Technically, we restrict EVI to work on the set of states $\calS_k^{\textsc{EVI}}$ that are optimistically reachable from the communicating set $\calScom_k$. 
In practice, $\calS_k^{\textsc{EVI}} = \calScom_k$ when $\mathcal{K}_k = \emptyset$ since all the transitions to $\calStrans_k$ are forbidden, otherwise $\calS_k^{\textsc{EVI}} = \calS$.
Alg.~\ref{alg:tevi} shows this variation of EVI that we name \emph{Truncated EVI}.
Then, at each episode $k$, \tucrl runs \tevi with the following parameters: $(\wt{g}_k, \wt{h}_k, \wt{\pi}_k) = \tevi(\boldsymbol{0}, \wb{\mathcal{M}}_k^+, \calS_k^{\textsc{EVI}}, \epsilon_k)$.
Starting from an initial vector $v_0 = 0$, TEVI iteratively applies (on a subset $\calS_k^{\textsc{EVI}}$ of states) the optimal Bellman operator $\wt{L}_{\wb{\mathcal{M}}_k^+}$ associated to the (extended) MDP $\wb{\mathcal{M}}_k^+$ defined as
 \begin{equation}\label{eq:Lop_m2plus}
 \begin{aligned}
         \forall v \in \mathbb{R}^S, \quad \wt{L}_{\wb{\mathcal{M}}_k^+}v(s) := \max_{a \in \A_s} \left\{
\max_{\wt{r} \in B_{r,k}(s,a)} \wt{r} +  (\wt{p}^{sa})^\transp v
\right\},
 \end{aligned}
 \end{equation}
 where $\wt{p}^{sa} = \argmax_{\wt{p} \in \wb{B}_{p,k}^+(s,a)} \{\wt{p}^T v\}$ can be solved using~\citep[][Fig. 2]{Jaksch10}, except for $(s,a) \notin \mathcal{K}_k$ for which we force $\wt{p}^{sa}(s') := 0$ for any $s' \in \calStrans_k$ (see Alg.~\ref{alg:opt.prob.h}). 
If TEVI is stopped when $sp_{\calS_k^{\textsc{EVI}}}\left\{v_{n+1} - v_n\right\} \leq \epsilon_k$ and the true MDP is sufficiently explored, then the greedy policy 
$\wt{\pi}_k := \pi_n$ \wrt $v_n$ is $\epsilon_k$-optimistic, \ie $\wt{g}_k := g_n \geq g^*_{M^*} - \epsilon_k$ (see Sec.~\ref{sec:sketch.regret.proof} for details). The policy $\wt{\pi}_k$ is then executed until the number of visits to a state-action pair is doubled or a new state is ``discovered'' (i.e., $s_{t} \in \calStrans_{k_t}$).
Note that the condition $sp_{\calS_k^{\textsc{EVI}}}\left\{v_{n+1} - v_n\right\} \leq \epsilon_k$ is always met after a \emph{finite} number of steps since the extended MDP $\wb{\mathcal{M}}_k^+$ is communicating on the restricted state space $\calS_k^{\textsc{EVI}}$.
Finally, notice that when the true MDP $M^*$ is communicating, there exists an episode $\wb{k}$ \st for all $k \geq \wb{k},~\calStrans_k = \emptyset$ and
\tucrl \emph{can be reduced to} \ucrl by considering $\mathcal{M}_k$ in place of $\wb{\mathcal{M}}_k^+$.
%

\begin{algorithm}[t]
        \caption{\textsc{Truncated Extended Value Iteration (Tevi)}}
        \begin{algorithmic}
                \STATE\textbf{Input:} value vector $v_0$, extended MDP $\mathcal{M}$, set of states $\wb{\calS}$, accuracy $\epsilon$
                \STATE\textbf{Output:} $g_n$, $v_n$, $\pi_n$
                \STATE $n := 0$
                \STATE $v_1(s) := \wt{L}_{\mathcal{M}} v_0(s) := \max_{a \in \A_s} \left\{\max_{\wt{r} \in B_{r}(s,a)} \wt{r} +  \max_{\wt{p} \in B_{p}(s,a)} \wt{p}^\transp v_0 \right\}$, $\forall s \in \wb{\calS}$ (see App.~\ref{app:algorithm})
                \WHILE{$\max_{s\in \wb{\calS}} \left\{ v_{n+1}(s) - v_{n}(s) \right\} -  \min_{s\in \wb{\calS}} \left\{ v_{n+1}(s) - v_{n}(s) \right\} > \epsilon$}
                    \STATE $n := n + 1$
                    \STATE $v_{n+1}(s) := \wt{L}_{\mathcal{M}} v_{n}(s)$, $\forall s \in \wb{\calS}$
                \ENDWHILE
                \STATE $g_n :=\frac{1}{2} \left(\max_{s\in \wb{\calS}} \left\{ v_{n+1}(s) - v_n(s) \right\} + \min_{s\in \wb{\calS}} \left\{ v_{n+1}(s) - v_{n}(s) \right\} \right)$
                \STATE $\pi_n(s) \in \argmax_{a \in \A_s} \left\{\max_{\wt{r} \in B_{r}(s,a)} \wt{r} +  \max_{\wt{p} \in B_{p}(s,a)} \wt{p}^\transp v_{n} \right\}$, $\forall s \in \wb{\calS}$
        \end{algorithmic}
        \label{alg:tevi}
\end{algorithm}


\section{Regret of \tucrl}\label{app:tucrl.regret}

We follow the proof structure of~\citet{Jaksch10,fruit2018constrained} and use similar notations. Nonetheless, several parts of the proof significantly differ from \citep{Jaksch10,fruit2018constrained}:
\begin{itemize}
 \item in Sec.~\ref{app:regret.Minset} we prove that after a finite number of steps, \tucrl is \emph{gain-optimistic} (which is not as straightforward as in the case of \ucrl),
 \item in Sec.~\ref{sec:bound_wt_delta} we show that the sums taken over the whole state space $\calS$ that appear in the main term of the regret decomposition of \ucrl can be restricted to sums over $\calScom_k$ thanks to the new stopping condition used for episodes and the use of the condition $N_k^\pm(s,a) > \sqrt{\sfrac{t_k}{SA}}$ (see~\eqref{eqn:restricted_summation}),
 \item in Sec.~\ref{sec:poorly.visited.states}, we bound the number of time steps spent in ``bad'' state-action pairs $(s,a)$ satisfying $N_k^\pm(s,a) \leq \sqrt{\sfrac{t_k}{SA}}$,
 \item in Sec.~\ref{app:regret.number.episodes}, we bound the number of episodes with the new stopping condition.
\end{itemize}

\subsection{Splitting into episodes}
The regret of \tucrl after $T$ time steps is defined as:
$
        \Delta(\tucrl, T) := Tg^* - \sum_{t=1}^T r_t(s_t,a_t) 
$.
Defining $\Delta_k = \sum_{s\in\calS, a \in \A} \nu_k(s,a) \left( g^* - r(s,a) \right)$ and using the same arguments as in \citep{Jaksch10,fruit2018constrained}, it holds with probability $1 - \frac{\delta}{12T^{4/5}}$ that:
\begin{align}\label{eqn:martingale_reward}
\begin{split}
        \Delta(\tucrl, T) \leq \sum_{k=1}^m \Delta_k + \rmaxbound\sqrt{\frac{5}{2}T\ln \left(\frac{8T}{\delta}\right)}
\end{split}
\end{align}

\subsection{Episodes with $M^* \in \mathcal{M}_k$}\label{app:regret.Minset}

We now assume that $M^* \in \mathcal{M}_k$. 
As done in App.~\ref{app:algorithm}, let's denote by $\wt{g}_k$, $\wt{h}_k$ and $\wt{\pi}_k$ the outputs of $\tevi(\bm{0}, \mathcal{M}_k^\diamond, \calS_k^{\textsc{EVI}}, \varepsilon_k)$ (see Alg.~\ref{alg:tevi}) where $\varepsilon_k := \rmaxbound/\sqrt{t_k}$ and
\begin{equation}
         \label{eq:tucrl.optimization.def}
        \calS_k^{\textsc{EVI}} = \begin{cases}
                \calScom_k & \text{if } \mathcal{K}_k = \emptyset \\
                \calS & \text{otherwise}
        \end{cases},
        \qquad~~ \mathcal{M}_k^\diamond = \begin{cases}
                \mathcal{M}_k=\wb{\mathcal{M}}_k & \text{if } \calStrans_k = \emptyset\\
                \wb{\mathcal{M}}_k^{+} & \text{otherwise}
        \end{cases}.
\end{equation}
\tevi returns an approximate solution of a slightly modified version of Problem~\ref{eq:tucrl.optimization}:
\[
(\wt{M}_k, \wt{\pi}_k) = \argmax_{M \in \mathcal{M}_k^\diamond, \pi} \{g^{\pi}_M\}.
\]
In order to bound $\Delta_k$ we first show that $\wt{g}_k \gtrsim g^*$ (up to $\rmaxbound/\sqrt{t_k}$-accuracy). If $\calStrans_k = \emptyset$ then by definition 
$\mathcal{M}_k^{\diamond} = \wb{\mathcal{M}}_k = \mathcal{M}_k \ni M^*$
and so we can use the same argument as in~\citep[][Sec. 4.3 \& Thm. 7]{Jaksch10}. If  $\calStrans_k \neq \emptyset$, the true MDP $M^*$ might not be ``included'' in the extended MDP $\wb{\mathcal{M}}_k^{+}$ considered by \evi and we cannot use the same argument. 
To overcome this problem we first assume that $t_k$ is big enough which allows us to prove a useful lemma (Lem.~\ref{L:M.com}):
\begin{equation}\label{eq:tkbig}
        t_k \geq \frac{2401}{9}\Big(\diamcom\Big)^2SA\left(\Strans_k \ln\left(\frac{2SAt_k}{\delta}\right)\right)^2:=C(k)
\end{equation}
where $\Strans_k:= \left| \calStrans_k \right|$ is the cardinal of $\calStrans_k$.
\begin{lemma} \label{L:M.com}
        Let episode $k$ be such that $M^* \in \mathcal{M}_k$, $\calStrans_k\neq \emptyset$ and ~\eqref{eq:tkbig} holds. Then,
        \[
                \left( \forall (s,a) \in \calScom_k \times \A, N_k^\pm(s,a) > \sqrt{\frac{t_k}{SA}} \right) \implies \calStrans_k = \calStrans
        \]
\end{lemma}
\begin{proof}       
        Assume that episode $k$ is such that~\eqref{eq:tkbig} holds 
        and that for any state-action pair $(s,a) \in \calScom_k\times \A$
        \[
                N_k^\pm(s,a) > \sqrt{\frac{t_k}{SA}} \geq \frac{49}{3} \diamcom\Strans_k \ln\left(\frac{2SAt_k}{\delta}\right)
        \]
        Since $\calStrans_k\neq \emptyset$ and $M^* \in \mathcal{M}_k$, for any $(s,a,s') \in \calScom_k \times \A \times \calStrans_k$
        \begin{align*}
                \underbrace{p(s'|s,a)}_{\text{transition probability in }M^*} &\leq \underbrace{\wh{p}_k(s'|s,a)}_{=0} + \beta_k^{sas'} = \underbrace{\sqrt{\frac{14 \wh{\sigma}_{p,k}^2(s'|s,a) \ln(2 SA t_k/\delta)}{N_k^+(s,a)}}}_{=0} +\frac{49 \ln(2 SA t_k/\delta)}{3 N_k^\pm(s,a)}\\ &\leq \frac{49\ln\left(2SAt_k/\delta\right)}{3N_k^\pm(s,a)}
                < \frac{1}{\diamcom \Strans_k} 
        \end{align*}
        where we have exploited the fact that $\wh{p}(s'|s,a) =0$ and $\wh{\sigma}_{p,k}^2(s'|s,a) = 0$ for any state $s'\in\calStrans_k$ (remember that $N_k(s,a,s') =0$).

        We denote by $\tau_{M^*}(s \to s')$ the shortest path between any pair of states $(s,s') \in \calS \times \calS$ in the true MDP $M^*$.
        Fix an arbitrary target state $\wb{s} \in \calStrans_k$ and denote by $\tau(s) := \tau_{M^*}(s \to \wb{s})$ and $\tau_{\min}:= \min_{s \in \calScom_k}\{ \tau(s)\}$. We have
        \begin{align*}
                \tau(\wb{s}) &= 0\\
                \forall s \in \calScom_k \quad 
                \tau(s) &= 1+ \min_{a\in\A_s} \left \{\sum_{s'\in\calS} \underbrace{p(s'|s,a) \tau(s')}_{\geq 0} \right\}
                \geq 1+ \min_{a\in\A_s} \left \{\sum_{s'\in\calScom_k} p(s'|s,a) \underbrace{\tau(s')}_{\geq \tau_{\min}} \right\}\\ 
                &\geq 1 + \tau_{\min} \cdot \min_{a\in\A} \left\{ \sum_{s' \in \calScom_k} p(s'|s,a) \right\}
                 = 1 + \tau_{\min} \cdot\min_{a\in\A} \left\{  1 - \sum_{s' \in \calStrans_k} \underbrace{p(s'|s,a)}_{< {\frac{1}{\diamcom \Strans_k}}}  \right\}\\
                & > 1 + \tau_{\min} \left( 1 - \sum_{s' \in \calStrans_k} \frac{1}{\diamcom \Strans_k} \right) = 1 + \tau_{\min} \left( 1 - \frac{1}{\diamcom } \right)
        \end{align*}
        Applying the above inequality to $\wt{s} \in \calScom_k$ achieving $\tau(\wt{s}) = \tau_{\min}$ yields $\tau_{\min} > \diamcom$.
        This implies that the shortest path in $M^*$ between any state $s \in \calScom_k \subseteq \calScom$ and any state in $\wb{s} \in \calStrans_k$ is strictly bigger than $\diamcom$ but by definition $\diamcom$ is the longest shortest path between any pair of states in $\calScom$. Therefore, $\wb{s} \in \calStrans$.
        Since $\wb{s} \in \calStrans_k$ was chosen arbitrarily, then $\calStrans_k = \calStrans$.
\end{proof}
As a consequence of Lem.~\ref{L:M.com}, under the assumptions that $M^* \in \mathcal{M}_k$, $\calStrans_k\neq \emptyset$ and ~\eqref{eq:tkbig} holds, there are only two possible cases:
\begin{enumerate}
        \item Either $\calStrans_k = \calStrans$,
        \item or $\exists (s,a) \in \calScom_k \times \A \;:\; N_k^\pm(s,a) \leq \sqrt{\frac{t_k}{SA}}$.
\end{enumerate}

\textbf{Case 1:} $\calStrans_k = \calStrans$ implies that $M^* \in \wb{\mathcal{M}}_k^{+}$.
This is because for any $(s,a,s') \in \calScom_k \times \A \times \calStrans_k$ we have $p(s'|s,a) = \wt{p}_k(s'|s,a) =  0$ and for any $(s,a,s') \notin \calScom_k \times \A \times \calStrans_k$ we have  $\left|p(s'|s,a) - \wh{p}_k(s'|s,a)\right| \leq \beta_{p,k}^{sas'} $ and so $p(\cdot|s,a) \in \wb{B}_{p,k}^{+}(s,a)$.
Since $M^* \in \wb{\mathcal{M}}_k^{+}$, we can use the same argument as \citet[Sec. 4.3 \& Theorem 7]{Jaksch10} to prove $\wt{g}_k \geq g^* - \frac{\rmaxbound}{\sqrt{t_k}}$.

\textbf{Case 2:} For any $(s,a) \in \calStrans_k \times \A $, 
$\wb{B}_{p,k}^{+}(s,a) = \mathcal{C}$
is the $(S-1)$-simplex denoting the maximal uncertainty about the transition probabilities, and $B_{r,k}(s,a) =[0, \rmaxbound]$.
We will now construct an MDP $M' \in \wb{\mathcal{M}}_k^{+}$ with optimal gain $\rmaxbound$.
For all $(s,a) \in \calStrans_k \times \A $, we set the transitions to $p_{M'}(s|s,a) = 1$ and rewards to $r_{M'}(s,a) = \rmaxbound$.
Let $(\wb{s},\wb{a})\in \calScom_k \times \A$ such that $N_k^\pm(\wb{s},\wb{a}) \leq \sqrt{\frac{t_k}{SA}}$ (which exists by assumption).
We set $p_{M'}(s'|\wb{s},\wb{a}) >0$ for all $s' \in \calStrans_k$. This is possible because by definition of $\wb{\mathcal{M}}_k^{+}$, the support of $p(\cdot|\wb{s},\wb{a})$ is not restricted to $\calScom_k$.
Finally, for all state-action pairs $(s,a) \in \calScom_k \times \A $, we set $p_{M'}(\wb{s}|s,a) >0$.
This is possible because by definition of $\wb{\mathcal{M}}_k^{+}$, the support of $p(\cdot|s,a)$ is only restricted to $\calScom_k$ and $\wb{s} \in \calScom_k$.
In $M'$, for all policies, all states in $\calStrans_k$ are absorbing states (\ie loop on themselves with probability 1) with maximal reward $\rmaxbound$ and all other states $s \in \calScom_k$ are transient.
The optimal gain of $M'$ is thus $\rmaxbound$ and since $M' \in \wb{\mathcal{M}}_k^{+}$ we conclude that $\wt{g}_k \geq \rmaxbound - \frac{\rmaxbound}{\sqrt{t_k}} \geq g^* - \frac{\rmaxbound}{\sqrt{t_k}}$.

In conclusion, \tevi is always returning an \emph{optimistic} policy when the assumptions of Lem.~\ref{L:M.com} hold. The regret $\Delta_k$ accumulated in episode $k$ can thus be upper-bounded as:
\begin{align*}
\Delta_k &=  \sum_{s,a}\nu_k(s,a)(g^* - r(s,a))  =  \sum_{s,a}\nu_k(s,a)(\underbrace{g^*}_{\mathclap{\leq \wt{g}_k + \frac{\rmaxbound}{\sqrt{t_k}}}} -\wt{r}_k(s,a)) 
+  \sum_{s,a}\nu_k(s,a)(\wt{r}_k(s,a) - r(s,a))\\ & \leq \underbrace{\sum_{s,a}\nu_k(s,a)(\wt{g}_k -\wt{r}_k(s,a))}_{:=\wt{\Delta}_k} +  \sum_{s,a}\nu_k(s,a)(\wt{r}_k(s,a) - r(s,a)) + \rmaxbound \sum_{s,a}\frac{\nu_k(s,a)}{\sqrt{t_k}}\end{align*}
To bound the difference between the optimistic reward $\wt{r}_k$ and the true reward $r$ we introduce the estimated reward $\wh{r}_k$:
\begin{align*}
 \forall s,a \in \calS \times \mathcal{A}, ~~\wt{r}_k(s,a) - r(s,a) = \underbrace{\wt{r}_k(s,a) -\wh{r}_k(s,a)}_{\leq \beta_{r,k}^{sa} ~\text{by construction}} + \underbrace{\wh{r}_k(s,a) - r(s,a)}_{\leq \beta_{r,k}^{sa} ~ \text{since} ~ M \in \mathcal{M}_k} \leq 2 \beta_{r,k}^{sa}
\end{align*}
and so in conclusion:
\begin{align}\label{eqn:regret_bound_decomposition_1}
\Delta_k \leq \wt{\Delta}_k + \underbrace{2\sum_{s,a}\nu_k(s,a) \beta_{r,k}^{sa} + \rmaxbound \sum_{s,a}\frac{\nu_k(s,a)}{\sqrt{t_k}}}_{:= U_k^1}
\end{align}

\subsection{Bounding $\wt{\Delta}_k$}\label{sec:bound_wt_delta}

The goal of this section is to bound the term $\wt{\Delta}_k := \sum_{s,a}\nu_k(s,a)(\wt{g}_k -\wt{r}_k(s,a))$.
We start by discarding the state-action pairs $(s,a) \in \mathcal{K}_k$ that have been poorly visited so far:
\begin{align}
 \wt{\Delta}_k 
 &= \sum_{s,a}\nu_k(s,a)(\wt{g}_k -\wt{r}_k(s,a))\underbrace{\mathbbm{1}\lbrace (s,a)\notin \mathcal{K}_k \rbrace}_{:= \mathbbm{1}_{k}(s,a)}\nonumber + \sum_{s,a}\nu_k(s,a)\underbrace{(\wt{g}_k -\wt{r}_k(s,a))}_{\leq \rmaxbound}\mathbbm{1}\lbrace (s,a)\in \mathcal{K}_k \rbrace \nonumber\\
 &\leq \underbrace{\sum_{s,a}\nu_k(s,a)(\wt{g}_k -\wt{r}_k(s,a))\mathbbm{1}_{k}(s,a)}_{:=\wt{\Delta}'_k}
 + \rmaxbound \sum_{s,a}\nu_k(s,a)\mathbbm{1}\lbrace (s,a)\in \mathcal{K}_k\rbrace 
 \label{eqn:wt_delta_bound_decomposition_1}
\end{align}
We will now bound the term $\wt{\Delta}_k' =\sum_{s}\nu_k(s,\wt{\pi}_k(s))(\wt{g}_k -\wt{r}_k(s,\wt{\pi}_k(s)))\mathbbm{1}_{k}(s,\wt{\pi}_k(s))$.
We recall that the policy $\wt{\pi}_k$ is obtained by executing $\tevi(0,\mathcal{M}_k^\diamond,\calS_k^{\textsc{EVI}},\varepsilon_k)$ (see Alg.~\ref{alg:tevi}) where $\varepsilon_k:= \rmaxbound/\sqrt{t_k}$ and $\calS_k^{\textsc{EVI}}$ and $\mathcal{M}_k^\diamond$ are defined in~\eqref{eq:tucrl.optimization.def}.
In all possible cases for both $\calS_k^{\textsc{EVI}}$ and $\mathcal{M}_k^\diamond$, this amounts to applying value iteration to a \emph{communicating} MDP with finite state space $\calS_k^{\textsc{EVI}}$ and \emph{compact} action space.
By~\citep[Thm. 8.5.6]{puterman1994markov}, since the convergence criterion of value iteration is met we have:
\begin{align}\label{eq:app.stopping.bound}
 \forall s \in \calS_k^{\textsc{EVI}}, ~~\left|  \wt{h}_k(s) + \wt{g}_k - \wt{r}_k(s,\wt{\pi}_k(s)) - \sum_{s'\in \mathcal{S}} \wt{p}_k(s'|s,\wt{\pi}_k(s)) \wt{h}_k(s') \right| \leq \frac{\rmaxbound }{\sqrt{t_k}}
\end{align}
For all $s \notin \calScom_k$, $\nu_k(s,\wt{\pi}_k(s)) =0$ due to the stopping condition of episode $k$.
Therefore we can plug~\eqref{eq:app.stopping.bound} in $\wt{\Delta}'_k$ and derive an upper bound restricted to the set $\calScom_k \subseteq \calS_k^{\textsc{EVI}}$. Before to do that, we further decompose $\wt{\Delta}'_k$ as:
\begin{align}\label{eqn:wt_delta_bound_decomposition_2}
\begin{split}
 \wt{\Delta}_k' 
 &\leq \sum_{s}\nu_k(s,\wt{\pi}_k(s))\left( \sum_{s'\in \mathcal{S}} \wt{p}_k(s'|s,\wt{\pi}_k(s)) \wt{h}_k(s') - \wt{h}_k(s) + \frac{\rmaxbound}{\sqrt{t_k}} \right)\mathbbm{1}_{k}(s,\wt{\pi}_k(s))\\
 &= \nu_k' \left( \wt{P}_k -I \right) \wt{h}_k + \rmaxbound \sum_{s,a}\frac{\nu_k(s,a)}{\sqrt{t_k}}\mathbbm{1}_{k}(s,a)
 \end{split}
\end{align}
where $\nu_k' = (\nu_k(s,\wt{\pi}_k(s))\mathbbm{1}_{k}(s,\wt{\pi}_k(s)))_{s\in \calS}$ is the vector of visit counts for each state and the corresponding action chosen by $\wt{\pi}_k$ multiplied by the indicator function $\mathbbm{1}_{k}$, $\wt{P}_k = (\wt{P}_k(s'|s,\wt{\pi}_k(s)))_{s,s'\in\calS}$ is transition matrix associated to $\wt{\pi}_k$ in $\wb{\mathcal{M}}_k^{+}$ and $I$ is the identity matrix. We now focus on the term $\nu_k'(\wt{P}_k -I)\wt{h}_k$. Since the rows of $\wt{P}_k$ sum to 1,
$
 \forall \lambda \in \Re, ~~ \big( \wt{P}_k -I \big) \wt{h}_k = \big( \wt{P}_k -I \big) \big( \wt{h}_k+\lambda e \big)
$
where $e = (1,\dots 1)^\intercal$ is the vector of all ones. Let's take $\lambda := -\min_{s\in \calScom_k}{\lbrace \wt{h}_k(s) \rbrace}$ and define $w_k := \wt{h}_k +\lambda e$ so that for all $s\in \calScom_k$, $w_k(s)\geq 0$ and $\min_{s\in \calScom_k}{\lbrace w_k(s)\rbrace} =0$. We have:
\begin{align*}
        \nu_k'(\wt{P}_k -I)\wt{h}_k &= \sum_{s\in \calS} \nu_k(s,\wt{\pi}_k(s))\mathbbm{1}_{k}(s,\wt{\pi}_k(s)) \left(\sum_{s'\in \calS}\wt{p}_k(s'|s,\wt{\pi}_k(s))w_k(s') - w_k(s) \right)
\end{align*}
We denote by $k_t := \sup \lbrace k \geq 1:~~ t_k \leq t\rbrace$ the current episode at time $t$. 
Whenever $s_t \in \calStrans_{k_t}$, episode $k_t$ stops before executing any action (see the stopping condition of \tucrl in Fig.~\ref{fig:ucrl.constrained}) implying that $\forall s \in \calStrans_k$, $\nu_k(s,\wt{\pi}_k(s)) = 0$. Therefore we have:
\begin{align*}
        \nu_k'(\wt{P}_k -I)\wt{h}_k &= \sum_{s\in \calScom_k} \nu_k(s,\wt{\pi}_k(s))\mathbbm{1}_{k}(s,\wt{\pi}_k(s)) \left(\sum_{s'\in \calS}\wt{p}_k(s'|s,\wt{\pi}_k(s))w_k(s') - w_k(s) \right)
\end{align*}
For all states $s$ such that $\mathbbm{1}_{k}(s,\wt{\pi}_k(s))=1$, \ie satisfying $N_k^\pm(s,\wt{\pi}_k(s))>\sqrt{\sfrac{t_k}{SA}}$, we force \tevi to set $\wt{p}_k(s'|s,\wt{\pi}_k(s)) =0$, $\forall s' \in \calStrans_k$, by construction of $\wb{\mathcal{M}}_k^{+}$ so that:
\begin{align}\label{eqn:restricted_summation}
        \nu_k'(\wt{P}_k -I)\wt{h}_k = \sum_{s\in \calScom_k} \nu_k(s,\wt{\pi}_k(s))\mathbbm{1}_{k}(s,\wt{\pi}_k(s)) \left(\sum_{s'\in \calScom_k}\wt{p}_k(s'|s,\wt{\pi}_k(s))w_k(s') - w_k(s) \right)
\end{align}
We can now introduce $p$:
\begin{align}
 \sum_{s'\in \calScom_k}
 \wt{p}_k(s'|s,\wt{\pi}_k(s))w_k(s') - w_k(s)
 &= \sum_{s'\in \calScom_k} \wt{p}_k(s'|s,\wt{\pi}_k(s))w_k(s') - p(s'|s,\wt{\pi}_k(s))w_k(s') \label{eqn:wt_p_minus_p}\\ 
 &\quad{} + \left(\sum_{s'\in \calScom_k}p(s'|s,\wt{\pi}_k(s))w_k(s') - w_k(s) \right)\label{eqn:p_minus_identity}
\end{align}
By definition $\calScom_k \subseteq \calScom$ and using $(1,\infty)$-H\"older's inequality , the term \eqref{eqn:wt_p_minus_p} can be bounded as
$
 \eqref{eqn:wt_p_minus_p} \leq \left\|\wt{p}_k(\cdot|s,\wt{\pi}_k(s)) - p(\cdot|s,\wt{\pi}_k(s)) \right\|_{1, \calScom} \cdot \max_{s' \in \calScom_k} \lbrace w_k(s')\rbrace
$
where for any vector $v\in\Re^\calS$, $\|v\|_{1,\calScom}:=\sum_{s \in \calScom}|v(s)|$.
Define $\overline{s} \in \argmax_{s\in\calScom_k}{\lbrace w_k(s) \rbrace}$ and $\wt{s} \in \argmin_{s\in\calScom_k}{\lbrace w_k(s) \rbrace}$.
By definition $\overline{s}, \wt{s} \in \calScom_k$ and $w_k(\wt{s})= \min_{s\in\calScom_k}{\lbrace w_k(s) \rbrace} =0$.
By Lem.~\ref{lem:span_shortest_path}, we know that for all $s,s'\in \calScom_k $,  the difference $w_k(s') - w_k(s) = \wt{h}_k(s') - \wt{h}_k(s)$ is upper bounded by $\rmaxbound \cdot\tau_{\wb{\mathcal{M}}_k^{+}}(s \rightarrow s')$.
We also know by Lem.~\ref{L:sp.bias_difference} that for all $s,s'\in \calScom_k,~\tau_{\mathcal{M}_k^{+}}(s \rightarrow s') =  \tau_{\wb{\mathcal{M}}_k^{+}}(s \rightarrow s')$. 
Since $M^*\in \mathcal{M}_k^{+}$ ($M^*$ is the true MDP), we also have that for all $s,s'\in \calScom_k \subseteq \calScom,~\tau_{\mathcal{M}_k^{+}}(s \rightarrow s') \leq \tau_{M^*}(s \rightarrow s') \leq \diamcom$. In conclusion, 
 $\forall s,s'\in \calScom_k,~~w_k(s') - w_k(s)\leq \rmaxbound \diamcom $
 and in particular
$
 \max_{s' \in \calScom_k} \lbrace w_k(s')\rbrace = w_k(\overline{s}) = w_k(\overline{s}) - w_k(\wt{s}) \leq \rmaxbound \diamcom.
$
Similarly to what we did to bound $|\wt{r}_k - r|$ \eqref{eqn:regret_bound_decomposition_1}, we bound the distance in $\ell_1$-norm between $\wt{p}_k$ and $p$ by introducing $\wh{p}_k$:
\begin{align}
        \left\|\wt{p}_k - p \right\|_{1,\calScom}
        \leq \left\|\wt{p}_k - \wh{p}_k \right\|_{1,\calScom}
        + \left\|\wh{p}_k - p \right\|_{1,\calScom}
        \leq 2 \left(\sum_{s' \in \calScom} \beta_{p,k}^{s\wt{\pi}_k(s)s'} \right) \label{eqn:wt_delta_bound_decomposition_3}
\end{align}

We now bound the contribution of the term~\eqref{eqn:p_minus_identity}. \citet{Jaksch10} decompose this term into a martingale difference sequence and a telescopic sum but due to the indicator function $\mathbbm{1}_k$, in our case the sum is not telescopic anymore and an additional term appears.
\begin{align}
 \eqref{eqn:p_minus_identity} &= \sum_{s\in \calS} \nu_k(s,\wt{\pi}_k(s))\mathbbm{1}_{k}(s,\wt{\pi}_k(s))\left(\sum_{s'\in \calScom_k}\underbrace{p(s'|s,\wt{\pi}_k(s))w_k(s')}_{\geq 0,~~ \forall s' \in \calScom_k}\mathbbm{1}{\lbrace s \in \calScom_k\rbrace} - w_k(s)\mathbbm{1}{\lbrace s \in \calScom_k\rbrace}\right)\nonumber\\
 &\leq \sum_{s\in \calS} \nu_k(s,\wt{\pi}_k(s))\mathbbm{1}_{k}(s,\wt{\pi}_k(s))\left(\sum_{s'\in \calScom_k}p(s'|s,\wt{\pi}_k(s))w_k(s') - w_k(s)\mathbbm{1}{\lbrace s \in \calScom_k\rbrace}\right)\nonumber\\
 &= \sum_{t=t_k}^{t_{k+1}-1} \left(\sum_{s'\in \calS_{k}^1}p(s'|s_t,\wt{\pi}_{k}(s_t))w_{k}(s')  - w_{k}(s_t)\mathbbm{1}{\lbrace s_t \in \calScom_k\rbrace} \right) \mathbbm{1}_{{k}}(s_t,\wt{\pi}_{k}(s_t))\nonumber\\
 &= \sum_{t=t_k}^{t_{k+1}-1} \underbrace{\left(\sum_{s'\in \calS_{k}^1}p(s'|s_{t},\wt{\pi}_{k}(s_{t}))w_{k}(s') - w_k(s_{t+1})\mathbbm{1}{\lbrace s_{t+1} \in \calScom_k\rbrace}\right) \mathbbm{1}_{k}(s_t,\wt{\pi}_k(s_t))}_{:=X_t} \label{eqn:mds}\\
 &\quad{}+ \underbrace{\sum_{t=t_k}^{t_{k+1}-1} \left(w_k(s_{t+1})\mathbbm{1}{\lbrace s_{t+1} \in \calScom_k\rbrace} - w_k(s_t)\mathbbm{1}{\lbrace s_{t} \in \calScom_k\rbrace} \right) \mathbbm{1}_{k}(s_t,\wt{\pi}_k(s_t))}_{\text{not telescopic due to }\mathbbm{1}_k!}\label{eqn:sum_not_telescopic}
\end{align}
Define the filtration $\mathcal{F}_{t} = \sigma(s_1,a_1, r_1,\dots, s_{t+1})$. Since $k_t$ is $\mathcal{F}_{t-1}$-measurable: 
\begin{align*}
 \mathbbm{E}\left[w_{k_t}(s_{t+1})\mathbbm{1}{\lbrace s_{t+1} \in \calScom_{k_t}\rbrace}\mathbbm{1}_{k_t}(s_t,\wt{\pi}_{k_t}(s_t))|\mathcal{F}_{t-1}\right] &= \underbrace{\sum_{s'\in \calScom_{k_t}}p(s'|s_{t},\wt{\pi}_{k_t}(s_{t}))w_{k_t}(s')\mathbbm{1}_{k_t}(s_t,\wt{\pi}_{k_t}(s_t))}_{\mathcal{F}_{t-1}-\text{measurable}}
\end{align*}
implying $\mathbbm{E}[X_{t}|\mathcal{F}_{t-1}]=0$ and so $(X_t,\mathcal{F}_t)_{t\geq 1}$ is a martingale difference sequence (MDS) with $|X_t| \leq \rmaxbound \diamcom$. We will bound~\eqref{eqn:mds} in the next section (Sec.~\ref{sec:regret_bound_sum}) using Azuma's inequality.
Using the fact that $\mathbbm{1}_{k}(s_t,\wt{\pi}_k(s_t)) = \mathbbm{1}\lbrace (s_t,\wt{\pi}_k(s_t))\notin \mathcal{K}_k\rbrace= 1 - \mathbbm{1}\lbrace (s_t,\wt{\pi}_k(s_t))\in \mathcal{K}_k\rbrace$ we can make a telescopic sum appear and rewrite \eqref{eqn:sum_not_telescopic} as:
\begin{align}
 \eqref{eqn:sum_not_telescopic} 
 &= \underbrace{\sum_{t=t_k}^{t_{k+1}-1} w_k(s_{t+1})\mathbbm{1}{\lbrace s_{t+1} \in \calScom_k\rbrace} - w_k(s_t)\mathbbm{1}{\lbrace s_{t} \in \calScom_k\rbrace}}_{=  w_k(s_{t_{k+1}})\mathbbm{1}{\lbrace s_{t_{k+1}} \in \calScom_k\rbrace} -  w_k(s_{t_k})\mathbbm{1}{\lbrace s_{t_k} \in \calScom_k\rbrace} \leq \rmaxbound \diamcom} ~~\text{(telescopic sum)} \nonumber\\
 &\quad{}+ \sum_{t=t_k}^{t_{k+1}-1} \underbrace{\left(w_k(s_t)\mathbbm{1}{\lbrace s_{t} \in \calScom_k\rbrace} -w_k(s_{t+1})\mathbbm{1}{\lbrace s_{t+1} \in \calScom_k\rbrace}\right)}_{\leq \rmaxbound \diamcom} \mathbbm{1}\lbrace (s_t,\wt{\pi}_k(s_t))\in \mathcal{K}_k\rbrace\nonumber\\
 &\leq \rmaxbound \diamcom + \rmaxbound \diamcom \sum_{s,a}\nu_k(s,a)\mathbbm{1}\lbrace (s_t,\wt{\pi}_k(s_t))\in \mathcal{K}_k\rbrace \label{eqn:wt_delta_bound_decomposition_4}
\end{align}
By gathering \eqref{eqn:wt_delta_bound_decomposition_1}, \eqref{eqn:wt_delta_bound_decomposition_2}, \eqref{eqn:wt_delta_bound_decomposition_3}, \eqref{eqn:mds} and \eqref{eqn:wt_delta_bound_decomposition_4} we obtain the following bound for $\wt{\Delta}_k$:
\begin{align}
        \wt{\Delta}_k
        & \leq 2 \rmaxbound \diamcom \sum_{s,a} \sum_{s'\in\calScom} \underbrace{\mathbbm{1}_{k}(s,a)}_{\leq 1}\underbrace{\nu_k(s,a) \beta_{p,k}^{sas'}}_{\geq 0}   
        + \sum_{t=t_{k}}^{t_{k+1}-1}X_t+\rmaxbound \diamcom \nonumber \\
        &\quad{}+ \rmaxbound (\diamcom+1) \sum_{s,a}\nu_k(s,a)\mathbbm{1}\lbrace N_k^\pm(s,a)\leq\sqrt{\sfrac{t_k}{SA}}\rbrace
        + \rmaxbound\sum_{s,a}\underbrace{\frac{\nu_k(s,a)}{\sqrt{t_k}}}_{\geq 0}\underbrace{\mathbbm{1}_{k}(s,a)}_{\leq 1} \nonumber \\ 
        \begin{split}\label{eqn:regret_bound_decomposition_2}
        &\leq 2 \rmaxbound \diamcom \sum_{s,a}\sum_{s' \in \calScom} \nu_k(s,a)\beta_{p,k}^{sas'}   + \rmaxbound (\diamcom+1) \sum_{s,a}\nu_k(s,a)\mathbbm{1}\lbrace (s,a) \in \mathcal{K}_k\rbrace \\ 
        &\quad{} + \sum_{t=t_{k}}^{t_{k+1}-1}X_t+\rmaxbound \diamcom + \rmaxbound \sum_{s,a}\frac{\nu_k(s,a)}{\sqrt{t_k}} := U_k^2
        \end{split}
\end{align}

\subsection{Summing over episodes with $M^* \in \mathcal{M}_k$ and $t_k \geq C(k)$}\label{sec:regret_bound_sum}

Denote by $\mathbbm{1}(k):=\mathbbm{1}\lbrace t_k \geq C(k)\rbrace \cdot \mathbbm{1}{\lbrace M^* \in \mathcal{M}_{k} \rbrace}$ the indicator function taking value $1$ only when both $M^* \in \mathcal{M}_{k}$ and $t_k \geq C(k)$. By gathering \eqref{eqn:regret_bound_decomposition_1} and \eqref{eqn:regret_bound_decomposition_2} we obtain:
\begin{align}
 \sum_{k=1}^{m}\Delta_k  \cdot \mathbbm{1}(k) \leq \sum_{k=1}^{m}\underbrace{U_k^1}_{\geq 0}  \cdot \underbrace{ \mathbbm{1}(k)}_{\leq 1} + \sum_{k=1}^{m}\underbrace{U_k^2}_{\geq 0}  \cdot \underbrace{ \mathbbm{1}(k)}_{\leq 1} \leq \sum_{k=1}^{m} U_k^1 + U_k^2
\end{align}
and so
\begin{align}
 \label{eqn:remaining_terms_regret}
 \sum_{k=1}^{m}\Delta_k  \cdot \mathbbm{1}(k) \leq 
         & 2  \sum_{k=1}^{m} \sum_{s,a} \nu_k(s,a)\left( \rmaxbound \diamcom \sum_{s' \in \calScom}\beta_{p,k}^{sas'} + \beta_{r,k}^{s,a} \right) + 2\rmaxbound \sum_{k=1}^{m}\sum_{s,a}\frac{\nu_k(s,a)}{\sqrt{t_k}}\\
 &+ \rmaxbound (\diamcom+1) \sum_{k=1}^{m}\sum_{s,a}\nu_k(s,a)\mathbbm{1}\lbrace (s,a) \in \mathcal{K}_k\rbrace+ \sum_{t=1}^{T}X_t \mathbbm{1}({k_t})  + \rmaxbound m \diamcom \nonumber
\end{align}

We will now upper-bound the terms appearing in \eqref{eqn:remaining_terms_regret}. The main novelty of \eqref{eqn:remaining_terms_regret} compared to \ucrl is the term $ \sum_{k=1}^{m}\sum_{s,a}\nu_k(s,a)\mathbbm{1}\lbrace (s,a) \in \mathcal{K}_k\rbrace$ which is not present in the proof of \citet{Jaksch10}. We will show in the next section that this term is bounded by $O(\sqrt{\calScom AT})$. All the other terms are similar to those found in \ucrl.

\subsubsection{Poorly visited state-action pairs}\label{sec:poorly.visited.states}

We first notice that by definition $t_{k_t}\leq t$ where $k_t := \sup \lbrace k \geq 1:~~ t_k \leq t\rbrace$ is the current episode at time $t$.
As a result,
\begin{align*}
  \mathbbm{1}\left\{  (s,a) \in \mathcal{K}_{k_t}\right\} := \mathbbm{1}\left\{  N_{k_t}^{\pm}(s_t,a_t) \leq \sqrt{\sfrac{t_{k_t}}{SA}}\right\} \leq \mathbbm{1}\left\{  N_{k_t}^{\pm}(s_t,a_t) \leq \sqrt{\sfrac{t}{SA}} \right\}
\end{align*}
Instead of directly bounding $\sum_{k=1}^{m}\sum_{s,a}\nu_k(s,a)\mathbbm{1}\lbrace (s,a) \in \mathcal{K}_{k}\rbrace$ we will bound the number of visits $Z_T$ in state-action pairs that have been visited less than $\sqrt{\sfrac{t}{SA}}$ times
\begin{align*}
        Z_T 
        := \sum_{t=1}^{T}\mathbbm{1}\left\{ N_{k_t}^\pm(s_t,a_t) \leq \sqrt{\sfrac{t}{SA}} \right\}
        \geq \sum_{k=1}^{m}\sum_{s,a}\nu_k(s,a)\mathbbm{1} \left\{ (s,a) \in \mathcal{K}_{k}\right\}
\end{align*}
Note that the quantity $N_{k}(s,a)$ is updated only after the end of episode $k$ and the stopping condition of episodes used by \tucrl implies that (see Fig.~\ref{fig:ucrl.constrained}):
\begin{align}\label{eqn:stopping_condition_ineuality}
 \forall k \geq 1, ~\forall (s,a) \in \calS \times \A, ~~ \nu_{k}(s,a) \leq N_k^+(s,a)
\end{align}
Moreover, for all $(s,a) \notin \calScom \times \mathcal{A}$, $\nu_k(s,a) = 0$ implying that only the states $s\in \Scom$ should be considered in the above sums.
Using \eqref{eqn:stopping_condition_ineuality}, we prove the following lemma:
\begin{lemma}\label{lem:bound_bad_state_action}
 For any $T \geq 1$ and any sequence of states and actions $\lbrace s_1, a_1, \dots \dots s_T, a_T \rbrace$ we have:
 \begin{align*}
    Z_T \leq 2\sqrt{\Scom AT} +2\Scom A.
 \end{align*}
\end{lemma}
\begin{proof}
For any episode $k$ starting at time $t_k$, and for any state-action pair $(s,a)$ we recall that $N_k(s,a)$ denotes the number of visits in $(s,a)$ prior to episode $k$ ($k$ not included) and by $\nu_k(s,a)$ the number of visits in $(s,a)$ during episode $k$: 
\begin{align*}
 N_{k}(s,a) := \sum_{t=1}^{t_k-1}\mathbbm{1}\lbrace (s_t,a_t)= (s,a) \rbrace ~~ &\text{and} ~~
 \nu_k(s,a) := \sum_{t=t_k}^{t_{k+1}-1}\mathbbm{1}\lbrace (s_t,a_t)= (s,a) \rbrace
\end{align*}
and so $_{k}(s,a) = \sum_{i=1}^{k-1}\nu_i(s,a)$.
By convention, we denote by $N_{k_T+1}(s,a):= \sum_{t=1}^{T}\mathbbm{1}\lbrace (s_t,a_t)= (s,a) \rbrace$ the total number of visits in $(s,a)$ after $T$ time steps ($T$ included).
 We first decompose $Z_T$ as:
 \begin{align*}
         &Z_T := \sum_{s,a} \sum_{t=1}^{T} \mathbbm{1}\Big\{ \max \lbrace 1, N_{k_t}(s,a)-1 \rbrace\leq \sqrt{\sfrac{t}{SA}}\Big\} \cdot \mathbbm{1}\Big\{ (s_t,a_t)= (s,a) \Big\} = \sum_{s \in \Scom} \sum_a Z_T(s,a)\\\
         &\text{where}~~ Z_T(s,a) :=  \sum_{t=1}^{T} \mathbbm{1}\Big\{ \max \lbrace 1, N_{k_t}(s,a)-1 \rbrace\leq \sqrt{\sfrac{t}{SA}}\Big\} \cdot \mathbbm{1}\Big\{ (s_t,a_t)= (s,a) \Big\}
 \end{align*}
 Using the fact that for all $t\geq 1$, $t_{k_t} \leq t \leq t_{k_t +1} -1$ we have:
 \begin{align}
         \forall T\geq \tau\geq 1, ~~ Z_{\tau}(s,a)
         &=\sum_{t=1}^{\tau} \underbrace{\mathbbm{1}\Big\{ \max \lbrace 1, N_{k_t}(s,a)-1 \rbrace\leq \sqrt{\sfrac{t}{SA}}\Big\}}_{\leq 1} \cdot \underbrace{\mathbbm{1}\lbrace (s_t,a_t)= (s,a)\rbrace}_{\geq 0} \nonumber \\
         &\leq \sum_{t=1}^{\tau} \mathbbm{1}\lbrace (s_t,a_t)=(s,a)\rbrace\nonumber \leq  \sum_{t=1}^{t_{k_\tau +1} -1} \mathbbm{1}\lbrace (s_t,a_t)= (s,a)\rbrace \\
         &= N_{k_{\tau} + 1}(s,a) \label{eqn:D_less_N}
 \end{align}
 Let's define $t_{s,a}$ as the last time that $Z_t(s,a)$ was incremented by $1$:
 \begin{align*}
         t_{s,a}&:= \max \Big\{ T\geq t\geq 1: \max \lbrace 1, N_{k_t}(s,a)-1 \rbrace\leq \sqrt{\sfrac{t}{SA}} ~~ \text{and} ~~ (s_t,a_t)= (s,a) \Big\} \\
                &~=\min \Big\{ T\geq t\geq 1: Z_{t}(s,a) = Z_T(s,a) \Big\} 
 \end{align*}
 We denote by $m_{s,a}:= k_{t_{s,a}}$ the corresponding episode. By definition,
 \begin{align}\label{eqn:expand_D_T_s_a}
  Z_T(s,a)=Z_{t_{s,a}}(s,a)
 \end{align}
 and
  \begin{align}\label{eqn:bound2_N_m_s_a}
  \max \lbrace 1, N_{m_{s,a}}(s,a)-1 \rbrace\leq \sqrt{\sfrac{t_{s,a}}{SA}}
 \end{align}
 Using \eqref{eqn:D_less_N} with $\tau = t_{s,a}$ we obtain:
 \begin{align}\label{eqn:bound_D_t_s_a}
  Z_{t_{s,a}} \leq N_{m_{s,a}+1}(s,a)
 \end{align}
 Moreover, by definition of $N_k(s,a)$ and \eqref{eqn:stopping_condition_ineuality}:
 \begin{align}
  &N_{m_{s,a}+1}(s,a) = N_{m_{s,a}}(s,a) + \underbrace{\nu_{m_{s,a}}(s,a)}_{\leq  N_{m_{s,a}}^+(s,a) } \leq 2  \underbrace{\max \lbrace 1, N_{m_{s,a}}(s,a) \rbrace}_{\mathclap \leq ~\max \lbrace 1, N_{m_{s,a}}(s,a) -1 \rbrace +1}   \nonumber\\
  &\implies  N_{m_{s,a}+1}(s,a) \leq 2\cdot  \max \lbrace 1, N_{m_{s,a}}(s,a) - 1 \rbrace + 2\label{eqn:bound1_N_m_s_a} 
 \end{align}
 Gathering \eqref{eqn:expand_D_T_s_a}, \eqref{eqn:bound2_N_m_s_a}, \eqref{eqn:bound_D_t_s_a}, and \eqref{eqn:bound1_N_m_s_a} we obtain:
 \begin{align*}
    Z_T(s,a) = {Z_{t_{s,a}}(s,a)} &\leq {\max \lbrace 1, N_{m_{s,a}+1}(s,a)-1\rbrace} + 1 \leq 2 \cdot {\max \lbrace 1, N_{m_{s,a}}(s,a) - 1\rbrace} +2\\
    &\leq 2 \sqrt{\sfrac{t_{s,a}}{SA}}+2 \\ &\leq  2 \sqrt{\sfrac{T}{SA}} +2 \\
    \implies Z_T &= \sum_{s \in \Scom} \sum_{a} Z_T(s,a) \leq 2 \sqrt{\Scom AT} + 2\Scom A
 \end{align*}
 where for the last inequality we used the fact that $\Scom \leq S$ (by definition) implying $\Scom/\sqrt{S} = \sqrt{\Scom/S} \cdot \sqrt{\Scom} \leq \sqrt{\Scom}$.
 This concludes the proof.
\end{proof}
As a consequence of Lem.~\ref{lem:bound_bad_state_action}:
\begin{align}\label{eqn:bound_bad_states}
        \sum_{k=1}^{m}\sum_{s,a}\nu_k(s,a)\mathbbm{1}\lbrace N_k^\pm(s,a)\leq\sqrt{\sfrac{t_k}{\Scom A}}\rbrace \leq Z_T \leq 2 \sqrt{\Scom AT} + 2\Scom A
\end{align}

\subsubsection{Confidence bounds $\beta_{r,k}^{sa}$ and $\beta_{p,k}^{sas'}$}

Since \eqref{eqn:stopping_condition_ineuality} holds, Lemma 19 of \citet{Jaksch10} can still be applied. 
Moreover, exploiting again the fact that for all $(s,a) \notin \calScom \times \mathcal{A}$, $\nu_k(s,a) = 0$ we obtain 
\begin{align}\label{eqn:nu_over_sqrt_N}
 \sum_{k=1}^{m}\sum_{s,a}\frac{\nu_k(s,a)}{\sqrt{t_k}} \leq \sum_{k=1}^{m}\sum_{s,a}\frac{\nu_k(s,a)}{\sqrt{N_k^+(s,a)}} \leq \left(\sqrt{2} +1 \right)\sqrt{\Scom AT}
\end{align}
and as shown in~\citep[Appendix F.7]{fruit2018constrained} (with the difference that $S$ is restricted to $\Scom$) we have:
\begin{align}\label{eqn:nu_over_N}
 \sum_{k=1}^{m}\sum_{s,a} \frac{\nu_k(s,a)}{N_k^\pm(s,a)} \leq 6\Scom A + 2\Scom A\ln(T)
\end{align}
The terms $ \sum_{k=1}^{m}\sum_{s,a}\nu_k(s,a)\beta_{r,k}^{sa}$ and $\sum_{k=1}^{m}\sum_{s,a, s'\in\calScom}\nu_k(s,a)\beta_{p,k}^{sas'}$ can then be bounded exactly as in \citep[App. F.7]{fruit2018constrained} with $S$ replaced by $\Scom$ (except in the logarithm).

\subsubsection{Number of episodes}\label{app:regret.number.episodes}
The stopping condition of episodes used by \tucrl (see Fig.~\ref{fig:ucrl.constrained}) combines the original stopping condition of \ucrl with the condition $s_t \in \calStrans_{k_t}$. Using only inequality \eqref{eqn:stopping_condition_ineuality}, \citet[Figure 1]{Jaksch10} proved that for any any sequence $\lbrace s_1,a_1, \dots, s_{T},a_T\rbrace$, the number of episodes is bounded by $1 +2SA + SA \log_2\left( \frac{T}{SA} \right)$.
Since \eqref{eqn:stopping_condition_ineuality} also holds in our case, the total number of episodes $m$ after $T$ time steps can be bounded by the same quantity (with $S$ replaced by $\Scom$ since sates in $\calStrans$ will never be visited) plus the number of times the event $s_t \in  \calStrans_{k_t}$ occurs. Since whenever $s_t \in \calStrans_{k_t}$ state $s_t$ is removed from $\calStrans_{k_t+1}$ and $s_t$ necessarily belongs to $\calScom$ (by definition), this event can happen at most $\calScom$ times. By Proposition 18 in~\citep{Jaksch10} we thus have:
\begin{align}\label{eqn:nb_episodes}
 m \leq 1 +2\Scom A + \Scom A \log_2\left( \frac{T}{\Scom A} \right) + \Scom
\end{align}

\subsubsection{Martingale Difference Sequence $X_t\cdot\mathbbm{1}({k_{t}})$}
In Sec.~\ref{sec:bound_wt_delta} we already proved that $(X_t,\mathcal{F}_t)_{t\geq 1}$ is an MDS \ie for all $t\geq 1$, $\mathbbm{E}[X_{t}|\mathcal{F}_{t-1}]=0$. Since $k_t$ is $\mathcal{F}_{t-1}$-measurable, we also have $\mathbbm{E}[X_{t}\mathbbm{1}({k_{t}})|\mathcal{F}_{t-1}]= \mathbbm{1}({k_{t}})\cdot\mathbbm{E}[X_{t}|\mathcal{F}_{t-1}] = 0$ with $|X_t\mathbbm{1}({k_{t}})| \leq \rmaxbound \diamcom$.
Therefore, $(X_t\mathbbm{1}({k_{t}}),\mathcal{F}_t)_{t\geq 1}$ is also an MDS. By Azuma's inequality (see for example \citep[Lemma 10]{Jaksch10}):
\begin{align}
\sum_{t=1}^{T} X_t\mathbbm{1}({k_{t}}) \leq  \rmaxbound \diamcom \sqrt{\frac{5}{2} T \ln\left(\frac{8T}{\delta}\right)} ~~ \text{w.p.}\geq 1-\frac{\delta}{12T^{5/4}} \label{eqn:bound_martingale}
\end{align}

\subsection{Completing the regret bound}

By gathering \eqref{eqn:remaining_terms_regret}, \eqref{eqn:bound_bad_states}, \eqref{eqn:nu_over_sqrt_N}, \eqref{eqn:nu_over_N}, \eqref{eqn:bound_martingale} and \eqref{eqn:nb_episodes} we conclude that with probability at least $1-\frac{\delta}{12T^{5/4}}$:
\begin{align}
 \begin{split}\label{eqn:final_bound_delta}
 \sum_{k=1}^{m}\Delta_k  \cdot \mathbbm{1}(k) \leq ~& 2 \left(\sqrt{28} +\sqrt{14} \right) \rmaxbound  \sqrt{ \Scom AT \ln\left(\frac{2SAT}{\delta} \right)} \left(\diamcom\sqrt{ (\nextstatescom-1)} +1\right) \\
                                                    &+ \frac{196}{3}\rmaxbound \Scom A \ln\left(\frac{2SAT}{\delta} \right)(3+ \ln(T))\left(\diamcom \Scom +1\right) \\
& + 2 \rmaxbound (\diamcom+1) (\sqrt{\Scom AT}+\Scom A)\\
 &+ \rmaxbound \diamcom \sqrt{\frac{5}{2} T \ln\left(\frac{8T}{\delta}\right)}  + 2\left(\sqrt{2} +1 \right)\rmaxbound\sqrt{\Scom AT} \\&+ \rmaxbound \diamcom \left(1 +2\Scom A + \Scom A \log_2\left( \frac{T}{SA} \right) + \Scom  \right)\\
 \leq& C\cdot\left(\rmaxbound \diamcom \sqrt{\nextstatescom \Scom AT \ln\left( \frac{SAT}{\delta} \right)} + \rmaxbound \diamcom \left(\Scom\right)^2 A \ln^2\left( \frac{SAT}{\delta} \right) \right)
 \end{split}
\end{align}
where $C$ is a numerical constant independent of the MDP instance.

From \eqref{eqn:martingale_reward}, with probability at least $1-\frac{\delta}{12T^{5/4}}$:
\begin{align*}
 \Delta(\tucrl,T) &\leq \sum_{k=1}^{m}\Delta_k + \rmaxbound \sqrt{\frac{5}{2} T \ln\left(\frac{8T}{\delta}\right)}\\& = \underbrace{\sum_{k=1}^{m}\Delta_k\mathbbm{1}(k)}_{\text{see~\eqref{eqn:final_bound_delta}}} + \sum_{k=1}^{m}\Delta_k\cdot(1-\mathbbm{1}(k)) + \rmaxbound \sqrt{\frac{5}{2} T \ln\left(\frac{8T}{\delta}\right)}
\end{align*}
where $1-\mathbbm{1}(k)$ is the complement of $\mathbbm{1}(k)$ \ie takes value $1$ only when either $t_k < C(k)$ (see \eqref{eq:tkbig} for the definition of $C(k)$) or $M^* \notin \mathcal{M}_{k}$. As is proved in Appendix F.2 of \citep{fruit2018constrained}, since both \eqref{eqn:stopping_condition_ineuality} and Theorem 1 of \citet{fruit2018constrained} hold, we have that with probability at least $1-\frac{\delta}{20T^{5/4}} \geq 1-\frac{\delta}{12T^{5/4}}$:
\begin{align}\label{eqn:failed_episodes}
 \sum_{k=1}^m \Delta_k \mathbbm{1}{\lbrace M^* \not\in \mathcal{M}_k \rbrace} \leq {\rmaxbound} \sqrt{T}
\end{align}
As a consequence of \eqref{eqn:stopping_condition_ineuality} $t_{k+1} \leq 2t_k$. Thus, by definition of the condition $t_k < C(k)$ we have 
\begin{align}\label{eqn:regret_bound_first_episodes}
 \sum_{k=1}^{m}\Delta_k\cdot\underbrace{\mathbbm{1}\lbrace t_k < C(k) \rbrace}_{\geq 0} &\leq 
  2 \rmaxbound C(k)  
  \leq \frac{4802}{9} \rmaxbound \left(\diamcom\right)^2S^3A  \ln^2\left(\frac{2SAT}{\delta}\right) 
\end{align}
Finally, by Boole's inequality: $1-\mathbbm{1}(k) \leq \mathbbm{1}\lbrace M^* \notin \mathcal{M}_{k} \rbrace +  \mathbbm{1}\lbrace t_k < C(k) \rbrace$
and so
\begin{align*}
 \sum_{k=1}^{m}\Delta_k\cdot(1-\mathbbm{1}(k)) \leq  \underbrace{\sum_{k=1}^{m}\Delta_k\cdot\mathbbm{1}\lbrace M^* \notin \mathcal{M}_{k} \rbrace}_{\text{see~\eqref{eqn:failed_episodes}}} + \underbrace{\sum_{k=1}^{m}\Delta_k\cdot\mathbbm{1}\lbrace t_k < C(k) \rbrace }_{\text{see~\eqref{eqn:regret_bound_first_episodes}}}
\end{align*}
In conclusion, there exists a numerical constant $C$ independent of the MDP instance such that for any MDP and any $T>1$, with probability at least $1 - \frac{\delta}{12T^{5/4}} - \frac{\delta}{12T^{5/4}}- \frac{\delta}{12T^{5/4}} = 1 - \frac{\delta}{4T^{5/4}}$ we have:
\begin{align}\label{eqn:final_bound_regret}
        \Delta(\tucrl,T) &\leq C\cdot\left(\rmaxbound \diamcom \sqrt{\nextstates \Scom A T \ln\left( \frac{SAT}{\delta} \right)} + \rmaxbound \left(\diamcom\right)^2 S^3 A \ln^2\left( \frac{SAT}{\delta} \right) \right)
\end{align}
Since $\sum_{T=2}^{+\infty}\frac{\delta}{4T^{5/4}}=  \delta$, by taking a union bound we have that the regret bound \eqref{eqn:final_bound_regret} holds with probability at least $1-\delta$ for all $T>1$.

\begin{algorithm}[t]
	\begin{algorithmic}
		\STATE \textbf{Input:} Probability estimate $\wh{p} \in \mathbb{R}^{n}$, confidence interval $\beta \in \mathbb{R}$, value vector $v\in\mathbb{R}^{n}$,
		subset of states $\mathcal{I} \subseteq \lbrace s_1, \dots, s_m \rbrace$, $m \leq n$, such that $\sum_{s \in \mathcal{I}}\wh{p}(s) =1$
		\STATE \textbf{Output:} Optimistic probabilities $\wt{p} \in \mathbb{R}^{n}$\\[.3cm]
		\STATE Let $\mathcal{I} = \{s_1, s_2, \ldots, s_{m}\}$ such that $v(s_1) \geq v(s_2) \geq \ldots \geq v(s_m)$
		\STATE $\wt{p}_1(s_1) = \min\left\{1, \wh{p}(s_1) + \frac{\beta}{2}\right\}$
		\STATE $\wt{p}_1(s_j) = \wh{p}(s_j), \quad \forall 1 < j \leq m$
		\STATE $j = m$
		\STATE $i = 1$
		\WHILE{$\sum_{s \in \mathcal{I}} \wt{p}_i(s) > 1$}
		\STATE $i = i + 1$
		\STATE $\wt{p}_i(s) = \wt{p}_{i-1}(s), \quad \forall s \neq s_j$  
		\STATE $\wt{p}_i(s_j) = \max \left\{0, 1 - \sum_{s \in \mathcal{I}\setminus \{s_j\}} \wt{p}_{i-1}(s) \right\} $
		\STATE $j = j - 1$
		\ENDWHILE
		\STATE $\wt{p}_i(s) := 0, ~~ \forall s \in \calS\setminus\mathcal{I}$
		\STATE $\wt{p} := \wt{p}_i$
	\end{algorithmic}
	\caption{\textsc{Optimistic Transition Probabilities (OTP)}~\citep{Jaksch10}}
	\label{alg:opt.prob.h}
\end{algorithm}
\section{Shortest Path Analysis}\label{app:shortestpath}

We are interesting in comparing the shortest path of any pair $(s, \wb{s}) \in \calS \times \calScom_k$ in $\mathcal{M}^{+}_k$ and $\wb{\mathcal{M}}^{+}_k$.
Formally, given a target state $\wb{s}$, the stochastic shortest path $\tau_{M}(s) := \tau_M(s \to \wb{s})$ of an (extended) MDP $M$ is the (negation) solution of the following Bellman equation
\begin{equation}\label{eq:ssp}
\begin{aligned}
        \tau_M(s) &= -1 + \max_{a \in \A_s, p \in B_p(s,a)} \left\{ p^\transp \tau_M \right\}, \qquad \forall s \neq \wb{s}\\
        \tau_M(\wb{s}) &= 0
\end{aligned}
\end{equation}

\subsection{Equivalence of Shortest Path in $\mathcal{M}_k^{+}$ and $\wb{\mathcal{M}}_k^{+}$}

We start by proving the following.

\begin{lemma}\label{L:sp.bias_difference}
	For any pair $(s,\wb{s}) \in \calS \times \calScom_k$,
	$\tau_{\mathcal{M}_k^{+}}(s \to \wb{s}) = \tau_{\wb{\mathcal{M}}_k^{+}}(s \to \wb{s})$.
\end{lemma}

        In order to analyse the properties of the stochastic shortest path we need to investigative the maximization over the confidence interval $B_p(s,a)$ either in $\mathcal{M}^{+}_k$ or $\wb{\mathcal{M}}^{+}_k$. This problem can be solved using Alg.~\ref{alg:opt.prob.h}. For any state-action pair $(s,a)$, we define $\wt{p}_{\mathcal{M}^{+}_k}(\cdot|s,a) = \textsc{OTP}(\wh{p}(\cdot|s,a), B_{p,k}^{+}(s,a), \tau, \calS)$ and $\wt{p}_{\wb{\mathcal{M}}^{+}_k}(\cdot|s,a) = \textsc{OTP}(\wh{p}(\cdot|s,a), \wb{B}_{p,k}^{+}(s,a), \tau, \calS)$.
        It is easy to notice that the optimistic probability vectors built by Alg.~\ref{alg:opt.prob.h} satisfy (either in $\mathcal{M}^{+}_k$ or in $\wb{\mathcal{M}}^{+}_k$)
        \begin{align*}
                \forall i \in \{1,\ldots, n\},\qquad &\wt{p}_i(s_1) \geq \wh{p}(s_1)\\
                \forall i \in \{2,\ldots, n\}, \forall l \in \{n-i+2, n\}, \qquad &\wt{p}_i(s_l) = \max\left\{0, 1 - \sum_{s' \neq s_l} \wt{p}_{i-1} (s') \right\}\\ 
                                                                                  & \;\qquad = \max\left\{0, \wh{p}(s_l) - \left(\sum_{s'} \wt{p}_{i-1}(s') - 1\right)\right\}\\
                                                                                  & \;\qquad \leq \wh{p}(s_l) 
        \end{align*}
        where $s_1, \ldots, s_n$ are such that $\tau(s_1) \geq \ldots \geq \tau(s_n)$.
        The algorithm may stop before $n$ iterations but this means that the states not processed are kept at $\wh{p}$. 

        We start considering the case in which $(s,a) \in \mathcal{K}_k$.
        Recall that $\forall s' \in \calStrans_k$, $\wh{p}(s'|s,a) =0$ by definition since $s'$ is not reachable from $\calScom_k$ (\ie $N_k(s,a,s') = 0$) and that $\mathcal{M}^{+}_k$ and $\wb{\mathcal{M}}^{+}_k$ consider the same empirical average for the transition probabilities (\ie $\wh{p}$).
        The shortest path to $\wb{s}$ is such that $\max_{s} \{\tau(s)\} = \tau(\wb{s}) = 0$ and $\tau(s) \leq -1$ for any state $s \in \calS \setminus \{\wb{s}\}$ (either in $\mathcal{M}^{+}_k$ or $\wb{\mathcal{M}}^{+}_k$).
        As a consequence, $s_1 = \wb{s}$ and for any $s' \in \calStrans_k$, 
        \[
                \wt{p}_{\mathcal{M}_k^+}(s') \leq \wh{p}(s') = 0, \text{ and } \wt{p}_{\wb{\mathcal{M}}_k^+}(s') \leq \wh{p}(s') = 0
        \]
        which ensures that  $\forall (s,a) \in \mathcal{K}_k$ the constraints in $\wb{\mathcal{M}}^{+}_k$ hold.
        This results is independent from the vector $v$ provided to \textsc{OTP}. Then, for any vector 
        $v \in V = \{v\in \mathbb{R}^S | v(\wb{s}) = 0 \wedge v(s) \leq -1, \; \forall s \in \calS \setminus \{\wb{s}\}\}$,
        we have that $\mathcal{I}^1 = \mathcal{I}^2$, since $\beta^1 = \beta^2$ and $\wt{p}_{\mathcal{M}_k^+}(s) = \wt{p}_{\wb{\mathcal{M}}_k^+}(s) =0$ for any $s \in \calS^{2}_k$ then: $ \wt{p}_{\mathcal{M}_k^+}(s') = \wt{p}_{\wb{\mathcal{M}}_k^+}(s'), ~\forall s' \in \calS$.
        Finally, $\forall (s,a) \in (\calS \times \A) \setminus \mathcal{K}_k$ it is trivial to notice that: $\forall s' \in \calS$, $\forall v \in V$, $\wt{p}_{\mathcal{M}_k^+}(s') = \wt{p}_{\wb{\mathcal{M}}_k^+}(s')$ since $B_{p,k}^{+}(s,a) = \wb{B}_{p,k}^{+}(s,a)$.

        The proof follows by noticing that $\tau_{\mathcal{M}^{+}_k} \in V$ and $\tau_{\wb{\mathcal{M}}^{+}_k} \in V$.

\subsection{Bounding the bias span}

\begin{lemma}\label{lem:span_shortest_path}
        Consider an (extended) MDP $M$ and define $L_M$ as the associated optimal (extended) Bellman operator.
        Given $h_0 = \boldsymbol{0}$, and $h_i = (L_M)^i h_0$ we have that        
        \[
                \forall s,s'\in \calS,~~ h_i(s') - h_i(s) \leq \rmaxbound \tau_M(s\rightarrow s')
        \]
        where $\tau_M(s\rightarrow s')$ is the minimum expected shortest path from $s$ to $s'$ in $M$.
\end{lemma}
\begin{proof}
        The proof follows from the application of the argument in~\citep[][Sec. 4.3.1]{Jaksch10}.
\end{proof}

{
\section{Tighter Regret Bound}
\label{sec:new.relaxation}
In this section we present a different relaxation of $\mathcal{M}_k$ that preserves the Bernstein nature of the confidence intervals (although the final regret bound is the same).
This relaxation makes the transition from \tucrl to \ucrl smooth when $\calStrans = \emptyset$ and may perform better empirically.
We initially introduced the relaxation using $\ell_1$-norm in order to prove the equivalence of the shortest paths (Lem.~\ref{L:sp.bias_difference}) implying that $sp_{\calScom_k}\{w_k\} \leq \diamcom$.
We now show that the same result (\ie $sp_{\calScom_k}\{w_k\} \leq \diamcom$) can be obtained by consider a perturbation of $B_{p,k}$ that preserves the Bernstein-like confidence intervals.

We start defining the new confidence set $\wb{Z}_p^k$ for any $(s,a,s') \in \calS \times \A \times \calS$ as
\begin{align}\label{eq:tucrl.bernstein_confidence_interval_p}
        \wb{Z}_{p,k}(s,a,s') := 
 \begin{cases}
                B_p^k(s,a,s') ~\text{if }  s \in \calStrans_k\\
                B_p^k(s,a,s') ~\text{if }  s\in\calScom_k ~\text{and }{p}_k^+(s,a) \geq \rho_{t_k}(s,a)\\
                \{0\} ~\text{if } s\in\calScom_k, ~{p}_k^+(s,a) < \rho_{t_k}(s,a), ~\text{and }s'\in \calStrans_k\\
            \left[\wh{p}_k(s'|s,a) - \beta_{p,k}^{sas'},\wh{p}_k(s'|s,a) + \beta_{p,k}^{sas'} + \zeta_{p,k}^{sa} \right] \cap \big[0,1\big]  ~\text{otherwise}
           \end{cases}
\end{align}
where for any $(s,a) \in \calS \times \A$
\begin{align}\label{eq:tucrl.bernstein_confidence_bound_p}
        \zeta_{p,k}^{sa} := \sum_{s'\in \calStrans_k} p_k^+(s'|s,a) = \Strans_k \cdot \underbrace{p_k^+(s,a)}_{:=\min\left\{1, \frac{49}{3} \frac{b_{k,\delta}}{N_k^{\pm}(s,a)} \right\}}
\end{align}
We then define $\wb{\mathcal{M}}_k^+ := \left\{ \calS, ~ \A,~ r_k(s,a) \in B_{r,k}(s,a), ~p_k(s'|s,a)\in \wb{Z}_{p,k}(s,a,s'), p_k(\cdot|s,a) \in \mathcal{C}  \right\}$.

It is possible to prove that 
\begin{lemma}
        For any pair $(s,\wb{s}) \in \calS \times \calScom_k$, $\tau_{\wb{\mathcal{M}}_k^+}(s \to \wb{s}) \leq \tau_{\wb{\mathcal{M}}_k}(s \to \wb{s})$. As a consequence, let $h_i = (L_{\wb{\mathcal{M}}_k^+})^i \boldsymbol{0}$, then 
        \[
                sp_{\calScom_k}\{h_i\} \leq \rmaxbound \max_{s,\wb{s} \in \calScom_k} \{\tau_{\wb{\mathcal{M}}_k^+}(s \to \wb{s})\} \leq \rmaxbound \max_{s,\wb{s} \in \calScom_k} \{\tau_{\wb{\mathcal{M}}_k}(s \to \wb{s})\} \leq \rmaxbound \diamcom
        \]
\end{lemma}
}

\section{Proof of Lem.~\ref{lem:weakly_communicating}}

 We prove the statement by contradiction: we assume that there exists a learning algorithm denoted $\mathfrak{A}_T$ satisfying
 \begin{enumerate}[leftmargin=*]
  \item for all $\varepsilon \in ]0,1]$, there exists ${T}_\varepsilon^\dagger \leq f(\varepsilon)$ such that $\mathbbm{E}[\Delta(\mathfrak{A}_T,M_\varepsilon,x,T)] < 1/6\cdot T$ for all $T \geq {T}_\varepsilon^\dagger$,
  \item there exists ${T}_0^* < +\infty$ such that $\mathbbm{E}[\Delta(\mathfrak{A}_T,M_0,x,T)] \leq C_2 (\ln(T))^\beta$ for all $T \geq {T}_0^*$.
  \end{enumerate}
 
 Any randomised strategy for choosing an action at time $t$ is equivalent to an (a priori) random choice from the set of all deterministic strategies. Thus, it is sufficient to show a contradiction when the action played by $\mathfrak{A}_T$ at any time $t$ is a deterministic function of the past trajectory $h_t:=\{s_1, a_1, r_1,\dots, s_t\}$. In the rest of the proof we assume that $\mathfrak{A}_T$ maps any sequence of observations $h_t=\{s_1, a_1, r_1,\dots, s_t\}$ to a (single) action $a_t$. 
 
 By trivial induction it is easy to see that as long as state $y$ has not been visited, the history $h_t$ is independent of $\varepsilon$ ($\mathfrak{A}_T$ can not distinguish between different values of $\varepsilon$ and plays exactly the same action when the past history is the same).
 
 Let's define $N_{T}^0(x,b) : = \sum_{t=1}^{T}\mathbbm{1}\{ (s_t, a_t)=(x,b) \}$ the number of visits in $(x,b)$ with $a_t = \mathfrak{A}_T(h_t)$ and $\varepsilon =0$. Note that $N_{T}^0(x,b)$ is not random since when $\varepsilon =0$ both action $b$ and action $d$ loop on $x$ with probability 1. For any $\varepsilon \in [0,1]$ and any horizon $T$ define the event:
 \begin{align*}
  F(T,\varepsilon) := \bigcap_{1\leq t\leq T} \left\lbrace s_t \neq y \right\rbrace
 \end{align*}
 where the sequence of states $s_t$ is obtained by executing $\mathfrak{A}_T$ on MDP $M_\varepsilon$. We will denote by $\overline{F(T,\varepsilon)}$ the complement of $F(T,\varepsilon)$. 
 
 For any horizon $T$, and independently of $\varepsilon$, there is only one possible trajectory $h_T = \{s_1, a_1, r_1,\dots, s_T\}$ that never goes to $y$ and which corresponds to the trajectory observed when $\varepsilon =0$. When $\varepsilon = 0$, the probability of this trajectory is $1$ and so $\mathbbm{P}\left(F(T,0)\right)=1$ (recall that everything is deterministic in this case) while in general we have:
 \begin{align}\label{eqn:proba_event_F}
  \forall T\geq 1,~\forall \varepsilon\in[0,1],~~\mathbbm{P}\left(F(T,\varepsilon)\right)= \left(1-\varepsilon \right)^{N_T^0(x,b)}
 \end{align}
  
 We now prove by contradiction that 
 \begin{align}\label{eqn:divergence_visits_b}
  \lim_{T\to +\infty}N_{T}^0(x,b) =+\infty
 \end{align}
 Let's assume that $C:= \max \left\lbrace 10, \max_{T\geq 1}\{N_{T}^0(x,b)\}\right\rbrace < +\infty$. Taking $\varepsilon=1/C$ and applying the law of total expectation we obtain:
 \begin{align*}
  \forall T \geq 1,~~\mathbbm{E}[\Delta(\mathfrak{A}_T,M_{1/C},x,T)] &= \underbrace{\mathbbm{E}\left[\Delta(\mathfrak{A}_T,M_{1/C},x,T)|F(T,1/C)\right]}_{= T/2 +1/2\cdot N_T^0(x,b) \geq T/2} \cdot \underbrace{\mathbbm{P}\left(F(T,1/C)\right)}_{=  \left(1-1/C \right)^{N_T^0(x,b)}}\\
  \\ &+ \underbrace{\mathbbm{E}\left[\Delta(\mathfrak{A}_T,M_{1/C},x,T)|\overline{F(T,1/C)}\right] \cdot \mathbbm{P}\left(\overline{F(T,1/C)}\right)}_{\geq 0}\\
  &\geq \frac{T}{2} \cdot \left(1-\frac{1}{C} \right)^{N_T^0(x,b)} \geq \frac{T}{2} \cdot \underbrace{\left(1-\frac{1}{C} \right)^{C}}_{\geq 1/3~~\text{by Lem.~\ref{lem:geom_bound}}} \geq  \frac{T}{6}
 \end{align*}
 where we used the fact that \begin{itemize}
                              \item $N_T^0(x,b) \leq C$  and $(1-1/C) \in [0,1]$ by definition, implying $\left(1-\frac{1}{C} \right)^{N_T^0(x,b)} \leq \left(1-\frac{1}{C} \right)^C$,
                              \item since $C\geq 10$ we have $\left(1-\frac{1}{C} \right)^{C} \geq 1/3$ by Lem.~\ref{lem:geom_bound} applied to $x = 1/C$,
                              \item and finally under event $F(T,1/C)$, the regret incurred is exactly $T/2 +1/2\cdot N_T^0(x,b) \geq T/2$.
                             \end{itemize}
  This contradicts our assumption that there exists $T^\dagger_{1/C} < +\infty$ such that for all $T\geq T^\dagger_{1/C}$, $~\mathbbm{E}[\Delta(\mathfrak{A}_T,M_{1/C},x,T)]< T/6$ and so \eqref{eqn:divergence_visits_b} holds.
  
  Since $\lim_{T\to +\infty}N_{T}^0(x,b) =+\infty$, it is possible to construct a strictly increasing sequence $(T_n)_{n\in \mathbb{N}}$ such that: \begin{align*}
																	    \forall n\in \mathbb{N},~N_{T_{n+1}}^0(x,b) > N_{T_{n}}^0(x,b),~~
                                                                                                                                            T_0 = {T}_0^*,~~ T_1 \geq C_2, ~~T_1 \geq C_2 (\ln(T_1))^\beta~~\text{and}~~ N_{T_1}^0(x,b) \geq 10
                                                                                                                                           \end{align*}
 We also define the (strictly decreasing) sequence: $\varepsilon_n := 1/N_{T_{n}}^0(x,b), ~\forall n \geq 1$. By the law of total expectation:
 \begin{align}
  \mathbbm{E}[\Delta(\mathfrak{A}_{T_n},M_{\varepsilon_n},x,T_n)] &= \underbrace{\mathbbm{E}\left[\Delta(\mathfrak{A}_{T_n},M_{\varepsilon_n},x,{T_n})|F({T_n},\varepsilon_n)\right]}_{\geq {T_n}/2} \cdot \underbrace{\mathbbm{P}\left(F({T_n},\varepsilon_n)\right)}_{=  \left(1-\varepsilon_n \right)^{N_{T_n}^0(x,b)}} \nonumber\\
  \nonumber\\ &+ \underbrace{\mathbbm{E}\left[\Delta(\mathfrak{A}_{T_n},M_{\varepsilon_n},x,{T_n})|\overline{F({T_n},\varepsilon_n)}\right] \cdot \mathbbm{P}\left(\overline{F({T_n},\varepsilon_n)}\right)}_{\geq 0}\nonumber\\
  &\geq \frac{T_n}{2} \cdot \left(1-\varepsilon_n \right)^{N_{T_n}^0(x,b)} = \frac{T_n}{2} \cdot \underbrace{\left(1-\varepsilon_n \right)^{1/{\varepsilon_n}}}_{\geq 1/3~~\text{by Lem.~\ref{lem:geom_bound}}} \geq  \frac{T_n}{6} \label{eqn:bound_contradiction1}
 \end{align}
 where we applied Lem.~\ref{lem:geom_bound} to $x = \varepsilon_n \leq 1/10$ since $N_{T_{n}}^0(x,b) \geq 10$ for all $n\geq 1$. Moreover, since by construction for all $n\geq 1$, $T_n >  T_0 = {T}_0^*$ we have by assumption that
 \begin{align*}
  \forall n\geq 1,~~&\mathbbm{E}[\Delta(\mathfrak{A}_{T_n},M_{0},x,T_n)] = \frac{1}{2}N_{T_{n}}^0(x,b) = \frac{1}{2 \varepsilon_n} \leq C_2 (\ln(T_n))^\beta\\
  &\implies T_n \geq \exp \left( \frac{1}{\left(2 C_2 \cdot \varepsilon_n \right)^{1/\beta}} \right)
 \end{align*}
 Since $\lim_{n\to +\infty} 1/\varepsilon_n = +\infty$ and $\lim_{x \to +\infty} \exp\left(x^{1/\beta}\right)/x^\alpha = +\infty$ there exists $N \in \mathbb{N}$ such that for all $n \geq N$, $T_n \geq f(\varepsilon_n)$. By assumption, for all $n \geq N$, \[\mathbbm{E}[\Delta(\mathfrak{A}_{T_n},M_{\varepsilon_n},x,T_n)] < \frac{T_n}{6}\]
 which contradicts \eqref{eqn:bound_contradiction1} therefore concluding the proof.

\begin{lemma}\label{lem:geom_bound}
 For all $x \in ]0,1/10]$, we have $(1-x)^{1/x} \geq 1/3$.
\end{lemma}
\begin{proof}
 It is easy to verify that the derivative of $x \longmapsto (1-x)^{1/x}$ is:
 \begin{align*}
 \forall x \in ]0,1/10],~~\frac{d}{dx}\left( (1-x)^{1/x} \right) = -\underbrace{\frac{(1-x)^{1/x-1} }{x^2}}_{\geq 0}\cdot\left( (1-x)\ln(1-x) +x \right)
 \end{align*}
 It is well known that for all $x\in]0,1[$, $x < -\ln(1-x) < \frac{x}{1-x}$ implying that $(1-x)\ln(1-x) +x$ is positive.
 Therefore, $\frac{d}{dx}\left( (1-x)^{1/x} \right)$ is negative on $]0,1/10]$ implying that $x \longmapsto (1-x)^{1/x}$ is decreasing. As a result:
 $
 \forall x \in ]0,1/10],~~(1-x)^{1/x} \geq 0.9^{10} > 1/3
 $.
\end{proof}

\section{Experiments - Three-State Domain}\label{app:experiments}
This domain was introduced in~\citep{fruit2018constrained} in order to show the inability of \ucrl to learn in weakly communicating MDPs.
The graphical representation of the domain is reported in Fig.~\ref{fig:toy3d}.
We keep the same means for the rewards (reported on Fig.~\ref{fig:toy3d}) but we change the distributions: uniform distributions with range $1/5$ instead of Bernouillis.
In the main paper we showed how the algorithms behave when $\delta =0$.
Here we consider the case the MDP is communicating by defining $\delta = 0.005$.
Fig.~\ref{fig:3states.five} shows that, as expected, \tucrl behaves similarly to \ucrl. 
In this example it is able to outperform \ucrl since the preliminary phase in which transitions to non-observed states are forbidden leads to a more conservative exploration that, due to the structure of the problem ($s_1$ is difficult to reach but it is also non-optimal), results in a smaller regret.

\begin{figure}[t]
        \begin{minipage}[b]{.49\textwidth}
        \centering
        \begin{tikzpicture}
                        \begin{scope}
	\tikzset{VertexStyle/.style = {draw, 
									shape          = circle,
	                                text           = black,
	                                inner sep      = 2pt,
	                                outer sep      = 0pt,
	                                minimum size   = 24 pt}}
	\tikzset{VertexStyle2/.style = {shape          = circle,
	                                text           = black,
	                                inner sep      = 2pt,
	                                outer sep      = 0pt,
	                                minimum size   = 14 pt}}
	\tikzset{Action/.style = {draw, 
                					shape          = circle,
	                                text           = black,
	                                fill           = black,
	                                inner sep      = 2pt,
	                                outer sep      = 0pt}}
	                                 
	\node[VertexStyle](s0) at (0,0) {$ s_{0} $};
	\node[Action](a0s0) at (.7,.7){};
	\node[VertexStyle](s1) at (2.5,0){$s_1$};
	\node[Action](a0s1) at (1.5,0.){};
	\node[VertexStyle](s2) at (5,0){$s_2$};
	\node[Action](a0s2) at (4.3,-0.7){};
	\node[Action](a1s2) at (5,1){};
    
	\draw[->, >=latex, double, color=red](s0) to node[midway, right]{{\small $a_0$}} (a0s0);
    \draw[->, >=latex](a0s0) to [out=30,looseness=0.8] node[midway, xshift=1em]{$\delta$} (s1);   
	\draw[->, >=latex](a0s0) to [out=30,in=120,looseness=0.8] node[above]{$1-\delta$} (s2);

	\draw[->, >=latex, double, color=red](s1) to node[above,xshift=0.2em]{{\small $a_0$}} (a0s1);
	\draw[->, >=latex](a0s1) to (s0);
    
	\draw[->, >=latex, double, color=red](s2) to node[midway, left]{{\small $a_0$}} (a0s2);
    \draw[->, >=latex](a0s2) to [out=210, in=330, looseness=0.8] node[midway,xshift=-1em]{$\delta$} (s1);   
	\draw[->, >=latex](a0s2) to [out=210, in=300, looseness=0.8] node[above]{$1-\delta$} (s0);
	\draw[->, >=latex, double, color=red](s2) to node[midway, right]{{\small $a_1$}} (a1s2);
	\draw[->, >=latex](a1s2) to [out=30, in=30, looseness=1.2] (s2);

    \node [above, yshift=0.2em] at (a0s0) {\scriptsize $r=0$};
    \node [below, xshift=-0.3em] at (a0s1) {\scriptsize $r =\frac{1}{3}$};
    \node [below, yshift=-0.2em, xshift=1em] at (a0s2) {\scriptsize $r =\frac{2}{3}$};
    \node [above] at (a1s2) {\scriptsize $r =\frac{2}{3}$};
    
    \node at ($(s1)+(0, -2cm)$) {};
    \end{scope}
    
\end{tikzpicture}
\caption{Three-state domain introduced in~\citep{fruit2018constrained}}
\label{fig:toy3d}
\end{minipage}\hfill
\begin{minipage}[b]{.49\textwidth}
                \centering
                \includegraphics[width=\textwidth]{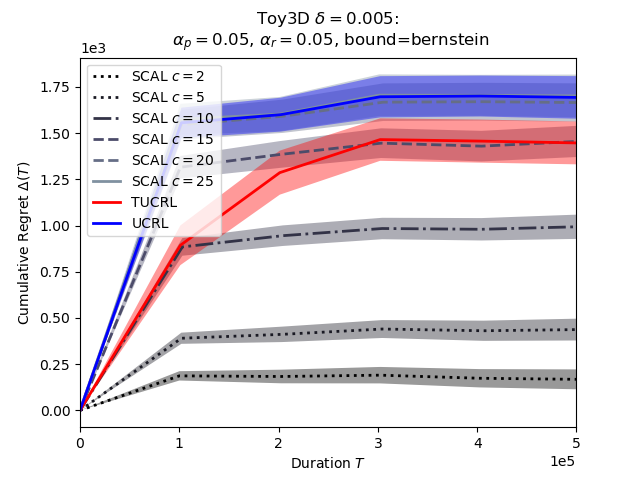}
                \caption{\label{fig:3states.five}Communicating three-state domain ($\delta=0.005$)}
\end{minipage}
\end{figure}

\end{document}